%% file: alt2023-main.tex
\documentclass[final, 12pt]{alt2023preprint} % for preprint submission
\usepackage{mathtools}
\usepackage{stmaryrd} % t 
\usepackage{tgtermes}  % t

\input{tex_src/preamble_common_alt.tex}
\title[Best-of-Both-Worlds Algorithms for Partial Monitoring]{Best-of-Both-Worlds Algorithms for Partial Monitoring}
\usepackage{times}
\altauthor{%
 \Name{Taira Tsuchiya} \Email{tsuchiya@sys.i.kyoto-u.ac.jp}\\
 \addr Kyoto University and RIKEN AIP
 \AND
 \Name{Shinji Ito} \Email{i-shinji@nec.com}\\
 \addr NEC Corporation
 \AND
 \Name{Junya Honda} \Email{honda@i.kyoto-u.ac.jp}\\
 \addr Kyoto University and RIKEN AIP
}

\begin{document}

\maketitle

\input{tex_src/0_abstract.tex}


\input{tex_src/1_introduction}

\input{tex_src/2_background}

\input{tex_src/3_ftrl}

\input{tex_src/4_local}

\input{tex_src/5_global}

% Acknowledgments---Will not appear in anonymized version
%\acks{We thank ....}

\bibliography{ref.bib}

\appendix

\input{tex_src/appendix}

\end{document}

%% file: tex_src/preamble_common_alt.tex
\SetCommentSty{mycommfont}
\usepackage{algorithm}  % for ALT
\usepackage{amsmath}
\usepackage{amssymb}
\usepackage{bbm}
\usepackage{bm}
\usepackage[normalem]{ulem}
\usepackage{booktabs} % for professional tables
\usepackage{eucal}  % good mathcal
\usepackage{mathrsfs}
\usepackage{comment}
\newenvironment{proofof}[1]%
{%
 \par\noindent{\bfseries\upshape Proof of {#1}. }%
}%
{\hfill\BlackBox\\[2mm]}
\newenvironment{proofsketchof}[1]%
{%
 \par\noindent{\bfseries\upshape Proof sketch of {#1}. }%
}%
{\hfill\BlackBox\\[2mm]}
\let\oldproofname=\proofname
\renewcommand{\proofname}{\rm\bf{\oldproofname}}
\SetArgSty{textnormal}
\SetKwInOut{Input}{Input}
\SetKwInOut{Output}{Output}
\SetKwComment{Comment}{$\triangleright$\ }{}

\def\e{\mathrm{e}}
\def\E{\mathbb{E}}

\def\R{\mathbb{R}}
\def\calA{\mathcal{A}}
\def\calB{\mathcal{B}}
\def\calC{\mathcal{C}}

\def\calE{\mathcal{E}}

\def\calG{\mathcal{G}}
\def\calH{\mathcal{H}}

\def\calL{\mathcal{L}}

\def\calR{\mathcal{R}}
\def\calP{\mathcal{P}}

\def\calT{\mathcal{T}}

\newcommand{\one}{\mbox{(i)}}
\newcommand{\two}{\mbox{(i\hspace{-.1em}i)}}

\DeclareMathOperator*{\argmin}{arg\,min}

\DeclareMathOperator*{\minimize}{minimize}

\newcommand{\ep}{\epsilon}

\newcommand{\ind}[1]{\mathbbm{1}{\left[#1\right]}}

\newcommand{\Expect}[1]{\mathbb{E}{\left[#1\right]}}

\newcommand{\innerprod}[2]{\left\langle{#1},{#2}\right\rangle}
\newcommand{\nn}{\nonumber\\}
\newcommand{\n}{\nonumber}
\newcommand{\per}{\,.}
\newcommand{\com}{\,,}

\usepackage{xcolor}
\DeclarePairedDelimiter{\abs}{\lvert}{\rvert} %
\DeclarePairedDelimiter{\brk}{[}{]}
\DeclarePairedDelimiter{\crl}{\{}{\}}
\DeclarePairedDelimiter{\prn}{(}{)}
\DeclarePairedDelimiter{\nrm}{\|}{\|}
\DeclarePairedDelimiter{\tri}{\langle}{\rangle}

\DeclareMathOperator{\tilTheta}{\tilde{\Theta}}

\newcommand{\prx}[1]{\left(#1\right)}

\newcommand{\brx}[1]{\left\{#1\right\}}

\newcommand{\lossmat}{\calL}
\newcommand{\fbmat}{\Phi}
\newcommand{\onemat}{\mathbf{1}}

\newcommand{\sumT}{\sum_{t=1}^T}

\newcommand{\sumak}{\sum_{a=1}^k}

\newlength{\subfigcolsep}
\newcommand{\ie}{\textit{i.e., }}
\newcommand{\eg}{\textit{e.g., }}

\usepackage{listings, color}
\definecolor{OliveGreen}{rgb}{0.0,0.6,0.0}
\definecolor{Orenge}{rgb}{0.89,0.55,0}
\definecolor{SkyBlue}{rgb}{0.28, 0.28, 0.95}
\newcommand{\linner}{\left\langle}
\newcommand{\rinner}{\right\rangle}

\newcommand{\bias}{\operatorname{bias}}
\newcommand{\opt}{\operatorname{opt}}
\newcommand{\at}{A_t}
\newcommand{\pt}{p_t}
\newcommand{\pa}{p_a}  % todo want to change index
\newcommand{\pta}{p_{t,a}}
\newcommand{\ptb}{p_{t,b}}
\newcommand{\ptAt}{p_{t,\at}}
\newcommand{\qt}{q_t}
\newcommand{\qa}{q_a}  % todo want to change index
\newcommand{\qta}{q_{t,a}}
\newcommand{\qtb}{q_{t,b}}

\newcommand{\dent}{d} % divergence induced by log-barrier
\newcommand{\cG}{c_\mathcal{G}}

\newcommand{\kpi}{k_{\Pi}}  % number of Pareto optimal actions

\newcommand{\Deltamin}{\Delta_{\min}}

%% file: tex_src/0_abstract.tex
\begin{abstract}
This study considers the partial monitoring problem with $k$-actions and $d$-outcomes and provides the first best-of-both-worlds algorithms, whose regrets are favorably bounded both in the stochastic and adversarial regimes. In particular, we show that for non-degenerate locally observable games, the regret is $O(m^2 k^4 \log(T) \log(k_{\Pi} T) / \Delta_{\min})$ in the stochastic regime and $O(m k^{2/3} \sqrt{T \log(T) \log k_{\Pi}})$ in the adversarial regime, where $T$ is the number of rounds, $m$ is the maximum number of distinct observations per action, $\Delta_{\min}$ is the minimum suboptimality gap, and $k_{\Pi}$ is the number of Pareto optimal actions. Moreover, we show that for globally observable games, the regret is $O(c_{\mathcal{G}}^2 \log(T) \log(k_{\Pi} T) / \Delta_{\min}^2)$ in the stochastic regime and $O((c_{\mathcal{G}}^2 \log(T) \log(k_{\Pi} T))^{1/3} T^{2/3})$ in the adversarial regime, where $c_{\mathcal{G}}$ is a game-dependent constant. We also provide regret bounds for a stochastic regime with adversarial corruptions. Our algorithms are based on the follow-the-regularized-leader framework and are inspired by the approach of exploration by optimization and the adaptive learning rate in the field of online learning with feedback graphs. 
%In addition, for the analysis of locally observable games, we refine the bound in exploitation by optimization for achieving the best-of-both-worlds guarantee. In addition, the algorithm in the globally observable case does not need to solve the convex optimization problem in every round and can be implemented efficiently.
\end{abstract}

%% file: tex_src/1_introduction.tex
\section{Introduction}

Partial monitoring (PM) is a general sequential decision-making problem with limited feedback, which can be seen as a generalization of the bandit problem.
A PM game $\calG = (\lossmat, \fbmat)$ is defined by the pair of a loss matrix $\lossmat \in [0,1]^{k \times d}$ and feedback matrix $\fbmat \in \Sigma^{k \times d}$, 
where $k$ is the number of actions, $d$ is the number of outcomes, and $\Sigma$ is a set of feedback symbols.
The game is sequentially played by a learner and opponent for $T \geq 3$ rounds.
At the beginning of the game, the learner observes $\calL$ and $\Phi$.
At every round $t \in [T]$, 
the opponent chooses an outcome $x_t \in [d]$, and then the learner chooses an action $\at \in [k]$, suffers an unobserved loss $\lossmat_{\at x_t}$, and receives a feedback symbol $\sigma_t = \fbmat_{\at x_t}$, where $\lossmat_{a x}$ is the $(a,x)$-th element of $\lossmat$.
In general, the learner cannot directly observe the outcome and loss, and can only observe the feedback symbol.
The learner's goal is to minimize their cumulative loss over all rounds.
The performance of the learner is evaluated by the regret $R_T$, which is defined as the difference between the cumulative loss of the learner and the single optimal action $a^*$ fixed in hindsight, that is,
%\begin{align}
$
  a^* = \argmin_{a \in [k]} \E \big[ \sumT \calL_{a x_t} \big]
$
and
$
  R_T
  =
  \E \big[ \sumT \prn[\big]{\lossmat_{\at x_t} - \lossmat_{a^* x_t}} \big]
  =
  \E \big[ \sumT \innerprod{\ell_{\at} - \ell_{a^*}}{e_{x_t}} \big]
  ,
%  \com
$
%$
%  R_T(a) 
%  = 
%  \E \big[ \sumT \prn[\big]{\lossmat_{\at, x_t} - \lossmat_{a,x_t}} \big]
%  =
%  \E \big[ \sumT \innerprod{\ell_{\at} - \ell_{a}}{e_{x_t}} \big]
%  \com
%$
%  \quad
%and
%\end{align}
% (\ie the action whose expected loss is the smallest).
where $\ell_a \in \R^d$ is the $a$-th row of $\lossmat$, and $e_x \in \{0,1\}^d$ is the $x$-th orthonormal basis of $\R^d$.

PM has been investigated in two regimes: the \textit{stochastic} and \textit{adversarial} regimes.
In the stochastic regime, outcomes $(x_t)_{t=1}^T$ are sampled from a fixed distribution $\nu^*$
in an i.i.d.~manner,
whereas in the adversarial regime, the outcomes are arbitrarily decided from the set of outcomes $[d]$ possibly depending on the history of the actions $(A_s)_{s=1}^{t-1}$.

Some of the first investigations on PM originate from work by~\citet{Rustichini99general, Piccolboni01FeedExp3}.
The seminal work was conducted by~\citet{CesaBianchi06regret, Bartok11minimax}, the latter of which showed that all PM games can be classified into four classes based on their minimax regrets.
They classified PM games into trivial, easy, hard, and hopeless games,
for which their minimax regrets are $0$, $\tilTheta(\sqrt{T})$, $\Theta(T^{2/3})$, and $\Theta(T)$, respectively.
The easy and hard games are also called~\textit{locally observable} and~\textit{globally observable} games, respectively.

PM algorithms have been established for both the stochastic and adversarial regimes.
In the adversarial regime, 
%many algorithms have been developed to achieve the aforementioned minimax regret, and 
the most common form of algorithms is an \textit{Exp3-type} one~\citep{auer2002nonstochastic}.
Recently, \citet{lattimore20exploration} showed that an Exp3-type algorithm with the approach of \textit{exploration by optimization} obtains the aforementioned minimax bounds.
Notably, they proved the regret bounds of 
$O(m k^{3/2} \sqrt{T \log k})$ 
%$O(\sqrt{T})$
for non-degenerate locally observable games, and 
$O((c_{\calG} T)^{2/3} (\log k)^{1/3})$ 
%$O(T^{2/3})$ 
for globally observable games,
where $m \leq \min\{|\Sigma|, d\}$ is the maximum number of distinct observations per action 
and $c_{\mathcal{G}}$ is a game-dependent constant defined in Section~\ref{sec:global}.
PM has also been investigated in the stochastic regime and
some algorithms exploiting the stochastic structure of the problem can achieve $O(\log T)$ regret bounds~\citep{Vanchinathan14BPM, Komiyama15PMDEMD, Tsuchiya20analysis}.

Algorithms assuming the stochastic model for losses can suffer linear regret in the adversarial regime, whereas algorithms for the adversarial regime tend to perform poorly in the stochastic regime.
Since knowing the underlying regime is difficult in practice, obtaining favorable performance for both the stochastic and adversarial regimes \textit{without} knowing the underlying regime is desirable.

%In the purpose of 
To achieve this goal, particularly in the classical multi-armed bandits, the Best-of-Both-Worlds (BOBW) algorithms that perform well in both stochastic and adversarial regimes have been developed.
The first BOBW algorithm was developed in a seminal paper by~\citet{bubeck2012best}, and the celebrated Tsallis-INF algorithm was recently proposed by~\citet{zimmert2021tsallis}.
BOBW algorithms have also been developed 
%for more general online-decision making problems 
beyond the multi-armed bandits:
(\eg\citealt{gaillard2014second, luo2015achieving, erez2021best, zimmert2019beating, lee2021achieving, jin2020simultaneously, huang22adaptive, saha22versatile}),
whereas such BOBW algorithms have never been investigated in PM.

Some BOBW algorithms are known to perform well also in the \textit{stochastic regime with adversarial corruptions}~\citep{lykouris2018stochastic}, which is an intermediate regime between the stochastic and adversarial regimes. 
This regime is advantageous in practice, since the stochastic assumption on outcomes is too strong whereas the adversarial assumption is too pessimistic.
Therefore it is also practically important to develop BOBW algorithms that cover this intermediate regime.

\begin{table}
  \caption{Regret upper bounds for PM.
%  stoc. is the stochastic regime, adv. is the adversarial regime, and corrup. is the stochastic regime with adversarial corruption.
  The constant $C \geq 0$ is the corruption level, and
  $\calR^{\mathrm{loc}}$ and $\calR^{\mathrm{glo}}$ are the regret upper bounds of the proposed algorithm in the stochastic regime for locally and globally games, respectively.
  ``observ.'' means observability.
  TSPM is the bound by~\citet{Tsuchiya20analysis}; refer to the paper for the definition of $\Lambda'$.
  ExpPM is by~\citet{lattimore20exploration}.
  }
  \label{table:regret_PM}
  \centering
  \small
%  \footnotesize
  \begin{tabular}{lllll}
    \toprule
    observ.
    & algorithm & stochastic  & adversarial (adv.) & adv. w/ corruptions \\
    \midrule
    locally
    &
    TSPM & $O\left( \frac{mk^2d \log(T)}{\Lambda'^2} \right)$  & --  & --  \\
    obs.
    &
    ExpPM & -- & $O(m k^{3/2} \sqrt{T \log k})$ & -- \\
    &
    \textbf{Proposed} &  $O\left(\frac{m^2 k^4 \log(T) \log(\kpi T)}{\Deltamin}\right)$ & $O(m k^{2/3} \sqrt{T \log(T) \log \kpi})$   & $\calR^{\mathrm{loc}} + \sqrt{C \calR^{\mathrm{loc}}}$  
     \\
    \midrule
    globally 
    & ExpPM & -- & $O((c_{\calG} T)^{2/3} (\log k)^{1/3})$ & -- \\
    obs.
    &
    \textbf{Proposed} & $O\left(\frac{c_{\mathcal{G}}^2 \log(T) \log(\kpi T)}{\Deltamin^2}\right)$ & $O((c_{\mathcal{G}} T)^{2/3} (\log(T)\log(\kpi T))^{1/3})$ & $\calR^{\mathrm{glo}} + (C^2 \calR^{\mathrm{glo}})^{1/3}$     
    \\
    \bottomrule
  \end{tabular}
\end{table}

\subsection{Contribution of This Study}
This study establishes new BOBW algorithms for PM based on the Follow-the-Regularized-Leader (FTRL) framework~\citep{mcmahan2011follow}.
We rely on two recent theoretical advances:  \one\, the Exp3-type algorithm for PM developed with the approach of exploration by optimization~\citep{lattimore20exploration}
and \two\, the adaptive learning rate for online learning with feedback graphs~\citep{ito22adversarially}, for which BOBW algorithms have been developed~\citep{erez2021best,ito2022nearly,rouyer2022near,kong22simultaneously}.
Note that it is known that the FTRL with the (negative) Shannon entropy regularizer corresponds to the Exp3 algorithm.

The regret bounds of the proposed algorithms are as follows.
We define 
the number of Pareto optimal actions by $\kpi \leq k$,
and the minimum suboptimality gap by $\Deltamin = \min_{a\in[k]} \Delta_a$,
where
$\Delta_a = (\ell_a - \ell_{a^*})^\top \nu^* \ge 0$ for $a \in [k]$ is the loss gap between action $a$ and optimal action $a^*$.
We show that for non-degenerate locally observable games, the regret is $O(m^2 k^4 \log(T) \log(\kpi T) / \Deltamin)$ in the stochastic regime and $O(m k^{2/3} \sqrt{T \log(T) \log \kpi})$ in the adversarial regime. %  where $\kpi \leq k$ is the number of Pareto optimal actions (defined in Section~\ref{sec:background}).
We also show that for globally observable games, the regret is  $O(c_{\mathcal{G}}^2 \log(T) \log(\kpi T) / \Deltamin^2)$ in the stochastic regime and $O((\cG T)^{2/3} (\log(T)\log(\kpi T))^{1/3})$ in the adversarial regime.
In addition, we also consider some intermediate regimes, such as the stochastic regime with adversarial corruptions~\citep{lykouris2018stochastic}, which we define in PM based on the corruptions on outcomes.
To our knowledge, the proposed algorithms are the first BOBW algorithms for PM.
Table~\ref{table:regret_PM} lists the regret bounds provided in this study and summarizes comparisons with existing work.

%This is the first regret bound for the FTRL-based algorithm for the stochastic globally observable PM games.
%while the bound for the adversarial regime is a factor of $(\log(T) \log(\kpi T) /\log k)^{1/3}$ worse than the algorithm by~\citet{lattimore20exploration} since their bound is $O(c_{\calG}^{2/3} (\log k)^{1/3} T^{2/3})$.
%We note that the stochastic regime with adversarial corruptions with $C = 0$ corresponds to the stochastic regime.
%We note that the bound~\eqref{eq:bound_global_arsbc} with $C = 0$ yields the bound in the stochastic regime.
% ==================================================
%We note that the bound for the stochastic regime with adversarial corruptions will also be provided.

\subsection{Technical Summary}
For locally observable games, we develop the algorithm based on the approach of \textit{exploration by optimization}~\citep{lattimore20exploration} with the Shannon entropy regularizer.
% (hereafter called \textit{negentropy}).
This approach is promising especially in locally observable games for bounding a component of regret, in which we consider a certain optimization problem with respect to the action selection probability. 
To obtain BOBW guarantees, we consider using a self-bounding technique~\citep{zimmert2021tsallis}. % , which is common in proving BOBW guarantees.
%In the self-bounding technique, we first derive the upper and lower bounds of regret using a random variable $P$ as $R(T) \leq O(\mathrm{polylog}(T)\sqrt{P})$ and $R(T) \geq O(P)$, respectively.
%Then we derive the (poly-)logarithmic regret by using these bounds as 
%$R(T) = 2 R(T) - R(T) \leq O(\mathrm{polylog}(T)\sqrt{P} - P) \leq O(\mathrm{polylog}(T))$.
In the self-bounding technique, we first derive upper and lower bounds of regret using a random variable depending on the action selection probability, and then derive a regret bound by combining the upper and lower bounds.
However, using the exploration by optimization may make some action selection probabilities extremely small, preventing deriving a meaningful lower bound.
To handle this problem, we consider an optimization over a \textit{restricted} feasible set. 
This restriction enables us to lower bound the regret such that the self-bounding technique is applicable, and we show that even with the optimization over the restricted feasible set, the component of regret is favorably bounded.
In addition, we consider the upper truncation of the learning rate developed by~\citet{ito2022nearly} to collaborate with the theory of exploration by optimization.

For globally observable games, we develop the algorithm using the Shannon entropy regularizer as for locally observable games.
To derive BOBW guarantees, we use the technique of adaptive learning rate developed in online learning with feedback graphs by~\citet{ito2022nearly}, but in a modified way. Their work uses a regularization called hybrid regularizers, which combines a Shannon entropy of the compensation of the action selection probability with typical regularizers~\citep{zimmert2019beating, ito22adversarially, ito2022nearly}. We think that naively applying this regularization also yields BOBW guarantees, but it loses the closed form of the action selection probability and requires solving an optimization problem each round.
This study shows that we can obtain the BOBW guarantee even only with the standard Shannon entropy regularization, and consequently, the proposed algorithm does not need to solve the optimization problem every round and can be implemented efficiently.

\subsection{Related Work}
In the adversarial regime, %  many algorithms have been developed and
FeedExp3 is an first Exp3-type algorithm, which has a first non-asymptotic regret bound~\citep{Piccolboni01FeedExp3} % , 
and is known to achieve a minimax regret of $O(T^{2/3})$~\citep{CesaBianchi06regret}.
Since then, Exp3-type algorithms have been used in many contexts.
\citet{bartok13near} relied on an Exp3-type algorithm as a subroutine of their algorithm.
\citet{lattimore19cleaning} showed that for a variant of the locally observable game (point-locally observable games), an Exp3-type algorithm achieves an $O(\sqrt{T})$ regret.
Recently, \citet{lattimore20exploration} showed that an Exp3-type algorithm using exploration by optimization can obtain bounds with good leading constants for both easy and hard games.
There are also a few algorithms that are not Exp3-type~\citep{Bartok11minimax,foster12no}.

PM has also been investigated in the stochastic regime, although less extensively than the adversarial regime~\citep{Bartok12CBP}.
One study~\citep{Komiyama15PMDEMD} is based on DMED~\citep{Honda11dmed}, in which the algorithm % mainly relies on the asymptotic argument and
 heavily exploits the stochastic structure, and the algorithm was shown to achieve an $O(\log T)$ regret with a distribution-optimal constant factor for globally observable games.
Two other approaches~\citep{Vanchinathan14BPM, Tsuchiya20analysis} are based on Thompson sampling~\citep{Thompson1933likelihood}.
They focus on another variant of locally observable games (strongly locally observable games), and the algorithms presented a strong empirical performance in the stochastic regime with an $O(\log T)$ regret bound~\citep{Tsuchiya20analysis}.  %  of $O(\log T)$ is known~\citep{Tsuchiya20analysis}.

%% file: tex_src/2_background.tex
\section{Background}\label{sec:background}
\paragraph{Notation}
Let $\nrm{x}$, $\nrm{x}_1$, and $\nrm{x}_\infty$ be the Euclidian, $\ell_1$-, and $\ell_\infty$-norms for a vector $x$ respectively,
and $\nrm{A}_\infty = \max_{i,j} \abs{A_{ij}}$ be the maximum norm for a matrix $A$.
Let $\calP_{k} = \{ p \in [0,1]^k : \nrm{p}_1 = 1 \}$ be the $(k-1)$-dimensional probability simplex.
A vector $e_{a} \in \{0,1\}^k$ is the $a$-th orthonormal basis of $\R^k$,
and $\onemat$ is the all-one vector.

\paragraph{Partial Monitoring}
Consider any PM game $\calG = (\lossmat, \fbmat)$.
Let $m \le |\Sigma|$ be the maximum number of distinct symbols in a single row of $\fbmat \in \Sigma^{k \times m}$ over all rows.
In the following, we introduce several concepts in PM.
Different actions $a$ and $b$ are \textit{duplicate} if $\ell_a = \ell_b$.
We can decompose possible distributions of $d$ outcomes in $\calP_d$ based on the loss matrix: % as follows.
%\begin{definition}[Cell decomposition]
%  For every action $a\in[k]$, 
%  \textit{cell}
%  $\calC_a = \crl{u \in \calP_{d} : \max_{b\in[k]} (\ell_a - \ell_b)^\top u \le 0}$
%  % $\Ci \coloneqq \{p \in \calP_M : (L_i - L_j)^\top p \leq 0,\,\forall j \neq i\}$ 
%  is the set of probability vectors in $\calP_d$ for which action $a$ is optimal.
%\end{definition}
for every action $a\in[k]$, \textit{cell}
$\calC_a = \crl{u \in \calP_{d} : \max_{b\in[k]} (\ell_a - \ell_b)^\top u \le 0}$
is the set of probability vectors in $\calP_d$ for which action $a$ is optimal.
Each cell is a convex closed polytope.
Let $\dim(\calC_a)$ be the dimension of the affine hull of $\calC_a$.
%Then we can define the Pareto optimality of the action as follows.
%\begin{definition}[Pareto optimality]
%  If $\calC_a = \emptyset$, action $a$ is \textit{dominated}.
%  For non-dominated actions, if $\dim(\calC_a) = d-1$ then action $a$ is \textit{Pareto optimal}, and if $\dim(\calC_a) < d-1$ then action $a$ is \textit{degenerate}.
%\end{definition}
If $\calC_a = \emptyset$, action $a$ is \textit{dominated}.
For non-dominated actions, if $\dim(\calC_a) = d-1$ then action $a$ is \textit{Pareto optimal}, and if $\dim(\calC_a) < d-1$ then action $a$ is \textit{degenerate}.
We denote the set of Pareto optimal actions by $\Pi$, 
and the number of Pareto optimal actions by $\kpi = |\Pi|$.
%Next, we define~\textit{neighbors} between two Pareto optimal actions, which is used to define the difficulty of PM games.
%\begin{definition}[Neighbors] % [Neighbors and neighborhood action]
%  Two Pareto optimal actions $a, b \in \Pi$ are \textit{neighbors} if $\dim(\calC_a \cap \calC_b) = d-2$.
%\end{definition}
Two Pareto optimal actions $a, b \in \Pi$ are \textit{neighbors} if $\dim(\calC_a \cap \calC_b) = d-2$,
and this notion is used to define the difficulty of PM games.
It is known that the undirected graph induced by the above neighborhood relations is connected 
(see \eg \citealt{Bartok12CBP}, \citealt[Lemma 37.7]{lattimore2020book}), and this is useful for loss difference estimations between distinct Pareto optimal actions.
A PM game is called non-degenerate if it has no degenerate actions.
From hereon, we assume that PM game $\calG$ is non-degenerate and contains no duplicate actions.
The following~\textit{observability} conditions characterize the difficulty of PM games.
\begin{definition} % [Local and global observability]
Neighbouring actions $a$ and $b$ are \textit{globally observable} if there exists function $w_e: [k] \times \Sigma \rightarrow \R$ such that
\begin{align}\label{eq:observability}
  \sum_{c=1}^k w_e(c, \fbmat_{c,x}) = \lossmat_{a,x} - \lossmat_{b,x} 
  \text{ for all } x \in [d]
  \per
\end{align}
Neighbouring actions $a$ and $b$ are \textit{locally observable} if there exists $w_e = w_{ab}$ satisfying~\eqref{eq:observability} and $w_e(c,\sigma) = 0$ for $c \not\in \{a,b\}$.
A PM game is called globally (resp.~locally) observable if all neighboring actions are globally (resp.~locally) observable.
\end{definition}
It is easy to see from the above definition that any locally observable games are globally observable, and this paper assumes that $\calG$ is globally observable.

\paragraph{Loss Difference Estimation}
Next, we introduce a method of loss difference estimations used in PM.
We recall that for globally observable games we can estimate loss differences between any Pareto optimal actions using~\eqref{eq:observability}.
Let $\calH = \{ G : [k] \times \Sigma \rightarrow \R^k \}$ be the class of loss estimators.
\begin{lemma}[Lemma 4 of~\citealt{lattimore20exploration}]\label{lem:Gdiff_Ldiff}
Consider any globally observable game. Then there exists a function $G \in \calH$ such that for all $b, c \in \Pi$, we have
\begin{align}
  \sumak (G(a, \fbmat_{ax})_b - G(a, \fbmat_{ax})_c) = \calL_{bx} - \calL_{cx}
  \ 
  \mbox{ for all }\ 
  x \in [d]
  \per
  \label{eq:Gdiff_Ldiff}
\end{align}
\end{lemma}
This result straightforwardly follows from the fact that the graph induced by the set of Pareto optimal actions is connected.
\citet{lattimore20exploration} provides the following example of $G$:
\begin{align}\label{eq:G_0}
  G(a, \sigma)_b = \sum_{e \in \mathrm{path}_\calT(b)} w_e(a, \sigma)
  \com
\end{align}
where $\mathrm{path}_\calT(b)$ is the set of edges from $b \in \Pi$
to an arbitrarily determined root $a \in \Pi$ on the connected graph over $\Pi$.

\paragraph{Intermediate Regimes between Stochastic and Adversarial Regimes}
Here, we discuss intermediate regimes between the stochastic and adversarial regimes: the stochastic regime with adversarial corruptions and an adversarial regime with a self-bounding constraint.

The stochastic regime with adversarial corruptions was originally considered by~\citet{lykouris2018stochastic} in the classical multi-armed bandits.
We define this regime in PM by considering the corruptions on the sequence of outcomes $(x_t)_{t=1}^T$.
In this regime, a temporary outcome $x'_t \in [d]$ is sampled from an unknown distribution $\nu^*$, and the adversary then corrupts $x'_t$ to $x_t$ without knowing $A_t$.
We define the corruption level by $C = \E\big[\sum_{t=1}^T \nrm{\calL e_{x_t} - \calL e_{x'_t}}_\infty \big] \geq 0$.
If $C = 0$, this regime corresponds to the stochastic regime, and if $C \ge T$, this regime corresponds to the adversarial regime.
As we will see, the proposed algorithms work without knowing the corruption level $C$. 
We also define another intermediate regime, a \textit{stochastically constrained adversarial regime}, in Appendix~\ref{sec:intermediate}.

%\todo{good intro?}
In this work, we consider an \textit{adversarial regime with a self-bounding constraint}, developed in the multi-armed bandits~\citep{zimmert2021tsallis} and includes the regimes that appeared so far.
\begin{definition}\label{def:ARSBC}
  Let $\Delta \in [0, 1]^k$ and $C \geq 0$.
  The environment is in an \textit{adversarial regime with a} $(\Delta, C, T)$ self-bounding constraint if it holds for any algorithm that
%  \begin{align}
%    \label{eq:defARSB}
$
    R_T \geq 
    \E \big[
      \sum_{t=1}^T \Delta_{A_t} - C
    \big]
%    \per 
$.
%  \end{align}
\end{definition}
We can show that the regimes that have appeared so far are included in the adversarial regime with a self-bounding constraint; the details are discussed in Appendix~\ref{sec:intermediate}.

In this study, we assume that there exists a unique optimal action.
This assumption has been employed by many studies aiming to develop BOBW algorithms~\citep{gaillard2014second,luo2015achieving,wei2018more,ito2021hybrid,zimmert2021tsallis}.

%% file: tex_src/3_ftrl.tex
\section{Follow-the-Regularized-Leader}\label{sec:ftrl}
This section introduces the FTRL framework and provides some fundamental bounds used in the analysis.
We recall that $\Pi$ is the set of Pareto optimal actions.
In the FTRL framework, a probability vector $\pt \in \calP_{k}$ over the action set $[k]$ is given as
\begin{align}\label{eq:def_q}
  \qt \in \argmin_{q \in \calP(\Pi)}
  \,
  \brk[\bigg]{
  \tri[\bigg]{\sum_{s=1}^{t-1} \hat y_s,\, q}
  +
  \psi_t(q)
  }
  \com
  \quad
  \pt = \calT_t(\qt)
  \com
\end{align}
where
the set $\calP(\calB) \coloneqq \{p \in \calP_k : p_a = 0 \mbox{ for } a \not\in \calB \}$ for $\calB \subset [k]$ is a convex closed polytope on the probability simplex with nonzero elements at indices in $\calB$,
$\hat y_s \in \R^k$ is an estimator of the loss at round $t$,
$\psi_t : \calP_k \rightarrow \R$ is a convex regularizer,
and
$\calT_t : \calP(\Pi) \rightarrow \calP_k$ is a map from $\qt$ to an action selection probability vector $\pt$.
%$\gamma_t \in [0, 1/2]$.
We use the Shannon entropy for $\psi_t$, which is defined as
\begin{align}\label{eq:defpsilocal}
  \psi_t(p) 
  = 
  \frac{1}{\eta_t} \sumak p_a \log(p_a) 
  = 
  - \frac{1}{\eta_t} H(p)
  \per
\end{align}
We can easily check that if we use the Shannon entropy with learning rate $\eta_t$,
% \ie, $\psi_t(p) = (1/\eta_t) \sumak p_a \log(p_a)$, 
$\qt \in \calP(\Pi)$ is expressed as
%has the following explicit form:
\begin{align}
  \qta 
  = 
  \frac{\ind{a \in \Pi} \exp\prx{-\eta_t \sum_{s=1}^{t-1} \hat{y}_{sa}}}
  {\sum_{b \in \Pi}\exp\prx{-\eta_t \sum_{s=1}^{t-1} \hat{y}_{sb}}}  
  \quad
  \mbox{ for }
  % \quad
  a \in [k]
  \per 
  \label{eq:q_for_entropy}
\end{align}
We set an estimator to $\hat{y}_t = G_t(A_t, \sigma_t)/\ptAt$ \citep{lattimore20exploration},
where for locally observable games, $G_t$ is obtained by minimizing a certain optimization problem, whereas for globally observable games $G_t$ is set to \eqref{eq:G_0}.
%Sections~\ref{sec:local} and~\ref{sec:global}.
%
The regret analysis of FTRL boils down to the evaluation of
$\sumT \sumak \pta (\hat y_{ta} - \hat y_{t a^*})$.
We can decompose this quantity into
\begin{align}
  \sumT \sumak \pta (\hat y_{ta} - &\hat y_{t a^*}) 
  \leq 
  \dashuline{
  \sumT
  \prn[\Big]{
  \psi_{t}(q_{t+1})
  -
  \psi_{t+1}(q_{t+1})
  }
  +
  \psi_{T+1} (e_{a^*}) 
  -
  \psi_1 (q_1)
  }
  \nn 
  &
  +
  \uwave{
  \sumT
  \prn[\Big]{
  \innerprod{\qt - q_{t+1}}{\hat{y}_t}
  -
  D_t(q_{t+1}, \qt)
  }
  }
  +
  \uline{
  \sumT \sumak (\qta - \pta) (\hat y_{ta} - \hat y_{t a^*})
  }
  \label{eq:lemFTRL}
  \com
\end{align}
where 
the inequality follows from the standard analysis of the FTRL framework~\citep[see \textit{e.g.,}][Exercise 28.12]{lattimore2020book}, 
and $D_t: \R^k \times \R^k \rightarrow \R_+$ is the \textit{Bregman divergence} induced by $\psi_t$, \ie 
$
  D_t(p, q) 
  = 
  \psi_t(p) - \psi_t(q) - \innerprod{\nabla \psi_t(q)}{p - q}
$.
We refer to the terms with 
\dashuline{dashed}, \uwave{wavy}, and \uline{straight} underlines in~\eqref{eq:lemFTRL}
as the \textit{penalty}, \textit{stability}, and \textit{transformation} terms, respectively.

We use a self-bounding technique to bound the regret in the stochastic regime, which requires a lower bound of the regret.
To this end, we introduce parameters $Q(a^*)$ and $\bar{Q}(a^*)$ given by
\begin{align}\label{eq:defQ}
  Q(a^*) = \sum_{t=1}^T (1 - q_{t,a^*})
%  \com
  \quad
  \mbox{and}
  \quad
  \bar{Q}(a^*)
  =
  \E \left[
  Q(a^*)
  \right]
  \per
\end{align}
Note that
$0 \le \bar{Q}(a^*) \leq T$ for any $a^* \in [k]$.
Based on quantity $\bar{Q}(a^*)$, the regret in the adversarial regime with a self-bounding constraint can be bounded from below as follows.
\begin{lemma}\label{lem:selfQ}
%In the stochastic regime, 
In the adversarial regime with a self-bounding constraint,
% (Definition~\ref{def:ARSBC}), 
%\fix{
if there exists $c \in (0,1]$ such that $\pta \geq c \, \qta$ for $t \in [T]$ and $a \in [k]$,
%}
the regret is bounded as % from below as
$
%\begin{align}
%  R_T \ge \frac{\Deltamin}{2} \bar{Q}(a^*)
  R_T \ge c\,\Deltamin \bar{Q}(a^*)
  -
  C
  \per
%\end{align}
$
\end{lemma}
All omitted proofs are given in Appendix~\ref{sec:proof}.
This lemma is used to derive poly-logarithmic regret bounds in the adversarial regime with a self-bounding constraint.

%% file: tex_src/4_local.tex
\section{Locally Observable Case}\label{sec:local}
This section provides a BOBW algorithm for locally observable games and derives its regret bounds.

\subsection{Exploration by Optimization in PM}
We first briefly explain the approach of exploration by optimization by~\citet{lattimore20exploration},
based on which our algorithm for locally observable games is developed.
The key idea behind the approach is to minimize a part of a regret upper bound of an Exp3-type algorithm (equivalently, FTRL with the Shannon entropy).
In particular, they consider the optimization on variables $G : [k] \times \Sigma \rightarrow \R^k$ and $p \in \calP_k$.
Their algorithm computes every round the function $G$ and the action selection probability vector $p$ by optimizing a part of the regret upper bound of FTRL, expressed as
\begin{align}\label{eq:opt}
  \minimize_{G \in \calH, \, p \in \calP_k}
  \quad
  \max_{x \in [d]} 
  \Bigg[
  \frac{(p-q)^\top \lossmat e_x}{\eta} 
  + 
  \frac{\bias_q(G ; x)}{\eta} 
  + 
  \frac{1}{\eta^2} \sumak 
  p_a 
  \innerprod{q}{\xi\left( \frac{\eta G(a, \fbmat_{ax})}{p_a} \right)}
  \Bigg] 
  \com
\end{align}
where
$\xi(x) = \e^{-x} + x - 1$
(we abuse the notation by applying $\xi$ in an element-wise manner),
and
\begin{align}\label{eq:defbias}
  \bias_q(G ; x)
  =
  \tri[\bigg]{q,\, \lossmat e_x - \sumak G(a, \fbmat_{ax})}
  +
  \max_{c\in\Pi}
  \prn[\bigg]{
  \sumak G(a, \fbmat_{ax})_c
  -
  \lossmat_{cx}
  }
  \com
\end{align}
is the bias function.
In the optimization problem~\eqref{eq:opt}, 
the first term corresponds to the transformation term,
the second term corresponds to the regret for using a biased estimator, and the third term comes from a part of the stability term.
Note that the bias term does not appear when $G$ satisfies~\eqref{eq:Gdiff_Ldiff}.
\subsection{Proposed Algorithm}
This section describes the proposed algorithm for locally observable games.
To obtain BOBW guarantees, we often rely on a self-bounding technique, which requires a certain lower bound on the action selection probability $p$~\citep{gaillard2014second, wei2018more, zimmert2021tsallis}. 
However, solving the optimization problem~\eqref{eq:opt} may result in $p_a = 0$ for a certain $a \in [k]$, which precludes the use of the technique.
The proposed algorithm considers the minimization problem over a restricted feasible set for $p$ instead of over $\calP_k$.
Let $\calP'_k(q)$ for $q \in \calP(\Pi)$ be $\calP'_k(q) = \{p\in\calP_k : p_a \geq q_a/(2k) \mbox{ for all } a \in [k] \} \subset \calP_k$.
We then consider the following optimization problem:
\begin{align}\label{eq:optprime}
  \minimize_{G \in \calH, \, p \in \calP'_k(q)}
  \quad
  \max_{x \in [d]} 
  \Bigg[
  \frac{(p-q)^\top \lossmat e_x}{\eta} 
  + 
  \frac{\bias_q(G ; x)}{\eta} 
  + 
  \frac{1}{\eta^2} \sumak 
  p_a 
  \innerprod{q}{\xi\left( \frac{\eta G(a, \fbmat_{ax})}{p_a} \right)}
  \Bigg] 
  \com
\end{align}
which implies that the solution $p$ of the optimization problem~\eqref{eq:optprime} satisfies $p \geq q/(2k)$.
This property is useful when applying the self-bounding technique to bound the regret in the stochastic regime (possibly with adversarial corruptions).
%We define the optimal value of the optimization problem~\eqref{eq:optprime} by $\opt'_q(\eta) \geq \opt_q(\eta)$ and its truncation at round $t$ by $V'_t = \max\{0, \opt'_{\qt}(\eta_t)\}$.
We define the optimal value of the optimization problem~\eqref{eq:optprime} by $\opt'_q(\eta)$ and its truncation at round $t$ by $V'_t = \max\{0, \opt'_{\qt}(\eta_t)\}$.

\paragraph{Regularizer and Learning Rate}
We use the Shannon entropy with learning rate $\eta_t$ in~\eqref{eq:defpsilocal} as a regularizer. %  in~\eqref{eq:def_q}.
The learning rate $\eta_t$ is defined as follows.
Let $\beta'_1 = c_1 \ge 1$ and 
\begin{align}
  \beta'_{t+1}
  =
  \beta'_{t}
  +
  \frac{ c_1 }{\sqrt{ 1 + (\log \kpi)^{-1} \sum_{s=1}^{t} H(q_s) }}
  \com
  \quad
  \beta_t = \max\left\{B, \beta'_t\right\} 
  \com
  \quad
  \mbox{and}
  \quad
  \eta_t = \frac{1}{\beta_t}
%  \per
  \label{eq:defeta}
%  \label{eq:defprebeta}
\end{align}
for $c_1 > 0$ (determined in Theorem~\ref{thm:locally_obs}).
%We define
%\begin{align}
%  \beta_t = \max\left\{B, \beta'_t\right\} 
%  \com
%  \quad
%  \mbox{and}
%  \quad
%  \eta_t = \frac{1}{\beta_t}
%  \per
%  \label{eq:defeta}
%\end{align}
The fundamental idea of this learning rate was developed by~\citet{ito2022nearly}, and we use its variant by the upper truncation of $\beta'_t$.
The truncation is required when applying the following lemma to bound $\opt'_q(\eta)$.
\begin{lemma}\label{lem:optprime_bound_local}
For non-degenerate locally observable games and $\eta \leq 1/(2mk^2)$, we have
\begin{align}
  \opt'_*(\eta) \coloneqq \sup_{q\in\calP_k} \opt'_{q} (\eta)
  \leq 3 m^2 k^3
  \per
  \n 
\end{align}
\end{lemma}
This lemma is a slightly stronger version of Proposition 8 of~\citet{lattimore20exploration}, in which the same upper bound is derived for the minimum value over larger feasible set $\calP_k \supset \calP'_k(q)$ in~\eqref{eq:opt} instead of~\eqref{eq:optprime}. Since the objective function of~\eqref{eq:opt} and~\eqref{eq:optprime} originally comes from a component of the regret, this lemma means that the restriction of the feasible set does not harm the regret bound.
%\begin{remark}
%In general $\opt_q \leq \opt'_q$ and what we know so far from Proposition 8 of~\citet{lattimore20exploration} is that the $\opt_q$ can be bounded from above.
%We are going to see in the new proof of~Lemma~\ref{lem:optprime_bound_local} is that the minimization op $p \in \calP_k$ in~\eqref{eq:opt} can be replaced with the minimization of $p \in \calP'_k(q)$ in~\eqref{eq:optprime}.
%\end{remark}
%}
Algorithm~\ref{alg:MixedExp3Locally} provides the proposed algorithm for locally observable games.
\LinesNumbered
\SetAlgoVlined
\begin{algorithm2e}[t]
\textbf{input:} $B$

\For{$t = 1, 2, \ldots$}{
Compute $\eta_t$ using~\eqref{eq:defeta} and
$\qt$ using~\eqref{eq:q_for_entropy}

Solve~\eqref{eq:optprime} with $\eta \leftarrow \eta_t$ and $q \leftarrow \qt$ to determine $V'_t = \max\{0, \opt'_{\qt}(\eta_t)\}$ and the corresponding solution $\pt$ and $G_t$

Sample $A_t \sim \pt$, observe $\sigma_t \in \Sigma$, compute $\displaystyle \hat{y}_t = {G_t(A_t, \sigma_t)}/{\ptAt}$, update $\beta'_t$ using~\eqref{eq:defeta} 
}
\caption{
BOBW algorithm for locally observable games
}
\label{alg:MixedExp3Locally}
\end{algorithm2e}

\subsection{Regret Analysis for Locally Observable Games}
With the above algorithm, we can prove the following regret bound for locally observable games.
\begin{theorem}\label{thm:locally_obs}
Consider any locally observable non-degenerate partial monitoring game.
If we run Algorithm~\ref{alg:MixedExp3Locally} with $B \ge 2 m k^2$ and 
$
c_1 = \Theta\prn[\big]{
  m k^{3/2} \sqrt{{(\log T)}/{(\log \kpi)}}
}
$,
we have the following bounds.
For the adversarial regime with a $(\Delta, C, T)$ self-bounding constraint, we have
\begin{align}
  R_T
  =
  O\prn[\Bigg]{
    \frac{m^2 k^4 \log(T) \log(\kpi T)}{\Deltamin}
    +
    \sqrt{
      \frac{C m^2 k^4 \log(T) \log(\kpi T)}{\Deltamin}
    }
  }
  \com
  \label{eq:bound_local_arsbc}
\end{align}
and for the adversarial regime, we have
\begin{align}
  R_T
  =
  O\prn[\big]{
    m k^{3/2} \sqrt{T \log(T) \log \kpi}
  }
  +
  2 m k^2 \log \kpi
  \per
  \n 
\end{align}
\end{theorem}
Note that~\eqref{eq:bound_local_arsbc} with $C = 0$ yields the bound in the stochastic regime.
The bound for the adversarial regime is a factor of $\sqrt{\log(T) \log(\kpi) / \log k}$ worse for large enough $T$ than the algorithm by \citet{lattimore20exploration}.
This comes from the difficulty of obtaining the BOBW guarantee,
where we need to aggressively change the learning rate when the environment looks not so much adversarial.
Note that exactly solving the optimization problem~\eqref{eq:optprime} is not necessary, and we discuss regret bounds for this case in Appendix~\ref{sec:eps_opt_bound}.
In the rest of this section, we provide a sketch of the analysis.

We start by decomposing the regret as follows. 
\begin{lemma}\label{lem:ftrl_bound_with_opt}
%If $\eta_t > 0$, we have
$
%\begin{align}
  R_T
%  &
  \leq 
  \Expect{
  \sumT
  \prx{
%    \frac{1}{\eta_{t+1}}
    \eta_{t+1}^{-1}
    -
%    \frac{1}{\eta_t}
    \eta_t^{-1}
  }
  H(q_{t+1})
  +
%  \frac{H(q_1)}{\eta_1}
  {H(q_1)}/{\eta_1}
  + 
  \sumT \eta_t V'_t
  }
  \per
%\end{align}
$
\end{lemma}
This can be proven by refining the analysis of the penalty term of Theorem 6 in~\citet{lattimore20exploration}, in which we rely on the standard analysis in~\eqref{eq:lemFTRL}, and the first and remaining terms correspond to the penalty term and the sum of the transformation and stability terms, respectively.
As will be shown in the proof of Theorem~\ref{thm:locally_obs}, the RHS of Lemma~\ref{lem:ftrl_bound_with_opt} can be bounded in terms of $\sumT H(q_t)$, for which we have the following bound.
\begin{lemma}\label{lem:entropy_bound}
For any $a^* \in [k]$, we have
%\begin{align}
$
    \sumT H(\qt) \le Q(a^*) \log (\e \kpi T/Q(a^*))
$.
%    \per
%\end{align}
\end{lemma}
%This lemma can be proven using $\qta = 0$ for $a \not\in \Pi$ with the proof of Theorem 5 of~\citet{ito2022nearly}.
We can show this lemma similarly to Lemma 4 of~\citet{ito2022nearly} by noting that $\qta = 0$ for $a \not\in \Pi$.
Finally, we are ready to prove Theorem~\ref{thm:locally_obs}.
Here, we only sketch the proof and provide the complete proof can be found in Appendix~\ref{subsec:proof_of_theorem_locally}.
%\begin{proofof}{Theorem~\ref{thm:locally_obs}}
\begin{proofsketchof}{Theorem~\ref{thm:locally_obs}}
%\todo{fix this to more sketch version}
We prove this theorem by bounding the RHS of Lemma~\ref{lem:ftrl_bound_with_opt}.
%Let $t_0 = \min\{t \in [T] : \beta'_t \geq B \}$.

\paragraph{[Bounding the penalty term]}
Since $\beta'_{t+1}$ is non-decreasing and $\beta'_t \le \beta_t$ from the definition of learning rate in~\eqref{eq:defeta}, it holds that
%$
\begin{align}
  &
  \sumT
  \prn[\big]{
    {\eta_{t+1}}^{-1}
    -
    {\eta_t}^{-1}
  }
  H(q_{t+1})
%  &
  \leq  
  \sumT
  \prn[]{
    \beta'_{t+1}
    -
    \beta'_t
  }
  H(q_{t+1})
  =
  \sum_{t=1}^T
  \frac{ c_1 \sqrt{\log \kpi} \,  H(q_{t+1})}{
    \sqrt{\log \kpi +  \sum_{s=1}^{t} H(q_s) }
  }
  \nn
%  &
%  =
%  {2 c_1 \sqrt{\log \kpi}} \cdot  
%  \sum_{t=1}^T
%  \frac{ H(q_{t+1}) }{
%    \sqrt{\log \kpi + \sum_{s=1}^{t} H(q_s) } + 
%    \sqrt{\log \kpi + \sum_{s=1}^{t} H(q_s) } 
%  }
%  \nn
  &
  \leq
  c_1 \sqrt{\log \kpi}
  \sum_{t=1}^T
  \frac{ {2 }  H(q_{t+1}) }{
    \sqrt{\sum_{s=1}^{t+1} H(q_s) } + 
    \sqrt{\sum_{s=1}^{t} H(q_s) } 
  }
%  \nn
%  &
%  =
%  {2 c_1 \sqrt{\log \kpi}}
%  \sum_{t=1}^T
%  \prn[\bigg]{
%    \sqrt{\sum_{s=1}^{t+1} H(q_s) } -
%    \sqrt{\sum_{s=1}^{t} H(q_s) } 
%  }
%  \nn
%  &
%  =
%  {2 c_1 \sqrt{\log \kpi}}
%  \prn[\Bigg]{
%    \sqrt{\sum_{s=1}^{T+1} H(q_s) } -
%    \sqrt{H(q_{1})} 
%  }
%  \nn
%  &
%  \leq 
%  {2 c_1 \sqrt{\log \kpi}}
%%  \prn[\Bigg]{
%    \sqrt{\sum_{s=2}^{T+1} H(q_s) } 
%%      -
%%      \sqrt{H(q_{1})} 
%%  }
%  \nn
%  &
  \leq 
  {2 c_1 \sqrt{\log \kpi}}
  \sqrt{\sum_{t=1}^{T} H(q_{t})}
  \com
  \label{eq:penalty_etadiffsum_bound_abbr}
\end{align}
%$
where 
the second inequality follows from $0 \le H(q_{t+1}) \le \log \kpi$, and
the last inequality follows by sequentially applying
$b/(\sqrt{a + b} + \sqrt{a}) = \sqrt{a + b} - \sqrt{a}$ for $a, b > 0$, the telescoping argument, $\sqrt{a+b} - \sqrt{b} \le \sqrt{a}$ for $a, b \ge 0$, and
%,and 
%the last inequality follows from 
$H(q_{T+1}) \leq H(q_{1})$.

\paragraph{[Bounding the sum of the transformation and part of stability terms]}
It holds that 
\begin{align}
  \sumT \eta_t V'_t
  \leq 
  \max_{s \in [T]} V'_s \sumT \eta_t
%  =
%  \prx{
%  \max_{s \in [T]} 
%  \max\brx{0,\, \opt'_*(\eta_s)}
%  }
%  \sumT \eta_t
  \leq 
  3 m^2 k^3
  \sumT \eta_t
  \leq 
  \frac{3 m^2 k^3 (1 + \log T)}{ c_1 }
  \sqrt{ 1 + 
%  (\log \kpi)^{-1} 
  \frac{1}{\log \kpi}
  \sum_{t=1}^{T} H(\qt) } 
  \com
  \label{eq:sum_etaV_bound_abbr}
\end{align}
where the second inequality follows from Lemma~\ref{lem:optprime_bound_local} and the last inequality follows since
%Next, we bound $\sumT \eta_t$.
%Using the definition of $\beta'_t$ in~\eqref{eq:defeta}, we can bound $\beta'_t$ as
the lower bound
%\begin{align}
$
  \beta'_t
  =
  c_1
  +
  \sum_{u=1}^{t-1}
  \frac{ c_1} {\sqrt{ 1 + (\log \kpi)^{-1} \sum_{s=1}^{u} H(q_s) }}
  \geq
  \frac{ c_1 t } {\sqrt{ 1 + (\log \kpi)^{-1} \sum_{s=1}^{t} H(q_s) } }  
%  \per
$
implies that
%\end{align}
%Using this inequality, we have
%$
\begin{align}
  \sumT \eta_t 
  \leq 
  \sum_{t=1}^T
  \frac{1}{\beta'_t}
  \leq 
  \sum_{t=1}^T
  \frac{1}{ c_1 t } 
  \sqrt{ 1 + \frac{1}{\log \kpi} \sum_{s=1}^{t} H(q_s) } 
%    &
  \leq
  \frac{1 + \log T}{ c_1 }
  \sqrt{ 1 + \frac{1}{\log \kpi} \sum_{t=1}^{T} H(\qt) } 
  \per
  \n 
  % \label{eq:sum_eta_bound_abbr}
\end{align}
%$

\paragraph{[Summing up arguments and applying a self-bounding technique]}
By bounding the RHS of Lemma~\ref{lem:ftrl_bound_with_opt} 
by~\eqref{eq:penalty_etadiffsum_bound_abbr} and~\eqref{eq:sum_etaV_bound_abbr}
with
$
  c_1 = \Theta\prn[\big]{
  m k^{3/2} \sqrt{{\log (T)}/{\log \kpi}}
  }
$, we have
$
%\begin{align}\label{eq:entropy_dep_bound_local}
  R_T 
%  &
%  \leq 
%  3 m^2 k^3 \E\Big[\frac{1 + \log T}{c_1} \sqrt{1 + (\log \kpi)^{-1} \sumT H(q_t)}\Big]
%  +
%  2 c_1 \sqrt{\log \kpi} \, \Expect{\sqrt{\sumT H(q_t)}}
%  +
%  \frac{\log \kpi}{\eta_1}
%  \nn
%  &
  =
  O\prn[\Big]{
    m k^{3/2} \allowbreak \sqrt{\log (T) \sumT H(q_t)} 
    +
    m k^{3/2} \sqrt{\log (T) \log \kpi}
  }
  + 
  2mk^2 \log \kpi
  \per 
%  \com
%\end{align}
$
%where we set
%%\begin{align}
%$
%  c_1 = \Theta\prn[\big]{
%  m k^{3/2} \sqrt{{\log (T)}/{\log \kpi}}
%  }
%$.
%  \per
%\end{align}
Since $\sumT H(q_t) \le T \log \kpi$, the desired bound for the adversarial regime is obtained.
We consider the adversarial regime with a self-bounding constraint in the following.
%Using Lemma~\ref{lem:entropy_bound}, if $Q(a^*) \le \e$, we have $\sumT H(q_t) \le \e \log(\kpi T)$ since $\kpi T \ge \e$, and otherwise we have $\sumT H(q_t) \le Q(a^*) \log(\kpi T)$.
%In the former case, we can trivially obtain the desired bound combined with~\eqref{eq:entropy_dep_bound_local}.
Here, we only consider the case of $Q(a^*) \geq \e$, since otherwise we easily obtain the desired bound. 
Note that Lemma~\ref{lem:entropy_bound} with $Q(a^*) \geq \e$ implies $\sumT H(q_t) \leq Q(a^*) \log(\kpi T)$.
Hence, for any $\lambda > 0$
%$
\begin{align} 
  R_T 
  &= 
  (1 + \lambda) R_T - \lambda R_T
  \leq
  \E\Big[
  (1 + \lambda)
  O\prx{ m k^{3/2} \sqrt{\log(T) \log(\kpi T) Q(a^*)} }
  -
  \frac{\lambda \Deltamin Q(a^*)}{2k}
  \Big]
  + \lambda C
  \nn
  &
  \leq 
  O\prn[\big]{
    \calR^{\mathrm{loc}}
    +
    \lambda \prn{ \calR^{\mathrm{loc}} + C }
    +
    {\calR^{\mathrm{loc}}}/{\lambda} 
  }
  \com
  \n 
\end{align}
%$
where 
the first inequality follows by Lemma~\ref{lem:selfQ} with $c = 1/(2k)$,
and the second inequality follows from 
$a \sqrt{x} - bx/2 \le a^2/(2b)$ for $a, b, x \geq 0$ and $\calR^{\mathrm{loc}} = {m^2 k^4 \log(T) \log(\kpi T)}/{\Deltamin}$.
Appropriately choosing $\lambda$ gives the desired bound.
%\end{proofof}
\end{proofsketchof}

%% file: tex_src/5_global.tex
\section{Globally Observable Case}\label{sec:global}
This section proposes an algorithm for globally observable games and derives its BOBW regret bound.
We use $G$ defined in~\eqref{eq:G_0} and let $c_{\calG} = \max\{1, k \nrm{G}_\infty \}$ be the game-dependent constant.
\subsection{Proposed Algorithm}
%This section provides the proposed algorithm for globally observable games.
%\todo{write some topic sentence}
%\paragraph{Regularizer and Learning Rate}
In the proposed algorithm for globally observable games,
we use the regularizer $\psi_t$ in~\eqref{eq:defpsilocal} as used in the locally observable case, but with different parameters.
We define $\beta_t, \gamma_t \in \R$ by
$\beta_1 = \max\{c_2, 2 c_{\calG}\}$
and
\begin{align}\label{eq:def_beta_gamma_global}
  \gamma'_{t}
  =
  \frac14
  \frac{c_1 b_t}{ c_1 +  \left( \sum_{s=1}^t b_s \right)^{1/3} },
  \quad
  \beta_{t+1}
  =
  \beta_t
  +
  \frac{c_2 b_t}{ \gamma_t' \left(
    c_1
    +
  \sum_{s=1}^{t-1} \frac{b_s a_{s+1}}{ \gamma_s' }
  \right)^{1/2} }
  \com
  \quad
  \gamma_{t}
  =
  \gamma'_{t}
  +
    \frac{c_{\calG}}{2\beta_t}
  \com
\end{align}
where $c_1$ and $c_2$ are parameters satisfying $c_1 \geq \max\{1, \log\kpi\}$, and
$a_t$ and $b_t$ are defined by
\begin{align}\label{eq:defatbtglobal}
  a_t = H(\qt) = - \sum_{a \in \Pi} \qta \log(\qta)
  \quad
  \mbox{and}
  \quad
  b_t = 1 - \max_{a \in \Pi} \qta 
  \per 
\end{align}
Note that we have $\psi_t(0) = 0$, and using $\beta_t \ge \beta_1 \ge 2 \cG$ and $b_t \le \sumak \qta \le 1$ we have
$
  \gamma_t 
  \le
  {c_1 b_t}/{(4c_1)} + {\cG}/{(2 \cG)}
  \leq 1/2
$.
We use the following transform from $\qt$ to $\pt$:
\begin{align}\label{eq:p_global}
  \pt = \calT(q_t) = (1 - \gamma_t) \qt + \frac{\gamma_t}{k} \onemat
  \per 
\end{align}
Algorithm~\ref{alg:MixedExp3Globally} presents the proposed algorithm for globally observable games.

\LinesNumbered
\SetAlgoVlined
\begin{algorithm2e}[t]
\For{$t = 1, 2, \ldots$}{
%Compute $\qt$ with the regularizer in~\eqref{eq:defpsiglobal} \\
Compute $\qt$ using~\eqref{eq:q_for_entropy}
%$\displaystyle \qta 
%= 
%\frac{\ind{a \in \Pi} \exp\prx{-\eta_t \sum_{s=1}^{t-1} \hat{y}_{sa}}}
%{\sum_{b \in \Pi}\exp\prx{-\eta_t \sum_{s=1}^{t-1} \hat{y}_{sb}}}$ % \\

Compute $a_t, b_t$ in~\eqref{eq:defatbtglobal}, % \\
%Compute 
$\gamma'_t, \gamma_t$ in~\eqref{eq:def_beta_gamma_global}, % ,  \\
and
%Compute 
$\pt$ from $\qt$ by~\eqref{eq:p_global}  % \\

Sample $\at \sim \pt$, observe $\sigma_t \in \Sigma$, compute $\displaystyle \hat{y}_t = {G(A_t, \sigma_t)}/{\ptAt}$, 
and update $\beta_t$ using~\eqref{eq:def_beta_gamma_global} 
}
\caption{
BOBW algorithm for globally observable games
}
\label{alg:MixedExp3Globally}
\end{algorithm2e}

\subsection{Regret Analysis for Globally Observable Games}
With the above algorithm, we can prove the following regret bound for globally observable games.
\begin{theorem}\label{thm:globally_obs}
Consider any globally observable partial monitoring game.
If we run Algorithm~\ref{alg:MixedExp3Globally} with
$
c_1 =
\Theta \prn[\big]{
	\prn[\big]{\cG^2 \log(T) \log(\kpi T)}^{1/3}
}
$
and
$
c_2 =
\Theta \prn[\big]{
  \sqrt{\cG^2 \log T}
}
$
,
we have the following bounds.
For the adversarial regime with a $(\Delta, C, T)$ self-bounding constraint, we have
\begin{align}
  R_T
%  =
%  O\prn[\Bigg]{
%    \frac{c_1^3}{\Deltamin^2}
%    +
%    \prn[\Bigg]{\frac{C^2 c_1^3}{\Deltamin^2} }^{1/3}
%  }
  =
  O\prn[\Bigg]{
    \frac{\cG^2 \log(T) \log(\kpi T)}{\Deltamin^2}
    +
    \prn[\Bigg]{\frac{C^2 \cG^2 \log(T) \log(\kpi T)}{\Deltamin^2} }^{1/3}    
  }
  \com
  \label{eq:bound_global_arsbc}
\end{align}
and for the adversarial regime, we have
\begin{align}
  R_T
  &
%  =
%  O\prn[\Bigg]{
%    c_1 T^{2/3}
%    +
%    \sqrt{\frac{\log \kpi}{\log(\kpi T)}}
%    c_1^{2/3}
%    T^{1/2}
%    +
%    \frac{c_1^2}{\sqrt{\log(\kpi T)}}
%    +
%    \max\{c_2, \,\cG\} \log \kpi
%  }
%  \nn
%  &
  =
  O\prn[\big]{
    \prn[\big]{\cG^2 \log (T) \log(\kpi T)}^{1/3} T^{2/3}
  }
  \n 
  \com
\end{align}
where in the last big-$O$ notation, 
%the terms with small orders with respect to $T$ are ignored.
the terms of $o(\mathrm{poly}(k, c_G) (T \log T)^{2/3})$ are ignored.
\end{theorem}
Note that~\eqref{eq:bound_global_arsbc} with $C = 0$ yields the bound in the stochastic regime.
The bound for the adversarial regime is a factor of $(\log(T) \log(\kpi T) /\log k)^{1/3}$ worse than the algorithm by \citet{lattimore20exploration}.
%\fix{
This comes from the difficulty of obtaining the BOBW guarantee,
where we need to aggressively change the learning rate when the environment looks not so much adversarial.
%where we need to adjust the learning rate so that the algorithm adopts to the \emph{easy} problem instances.
%Nevertheless, to our knowledge, the proposed algorithm is the first BOBW algorithm for globally observable games.
%}

We begin the analysis by decomposing the regret as follows. 
\begin{lemma}\label{lem:FTRL}
The regret of Algorithm~\ref{alg:MixedExp3Globally} is bounded as 
%The regret of FTRL in the form of Algorithm~\ref{alg:MixedExp3Globally} is bounded as 
% ~\eqref{eq:p_global} can be bounded as
%\begin{align}
%  &
$
  R_T
  \leq 
  \E\big[
    \sumT
    \gamma_t
%  \big]
  +
%  \E\big[
    \sumT
    \prn[\big]{
    \innerprod{\hat{y}_t}{\qt - q_{t+1}}
    -
    D_t(q_{t+1}, \qt)
    }
%  \big]
  +
%  \E\big[
    \sumT
    \prn[\big]{
    \psi_{t}(q_{t+1})
    -
    \psi_{t+1}(q_{t+1})
    }
    +
    \psi_{T+1} (e_{a^*}) 
    -
    \psi_1 (q_1)
  \big]
  \per
$
%\end{align}
\end{lemma}
This lemma can be proven based on the fact that we can estimate loss differences between Pareto optimal actions, and boundedness of $\calL$, combined with the standard analysis of FTRL given in~\eqref{eq:lemFTRL}.
Note that the first, second, and last terms correspond to the transformation, stability, and penalty terms, respectively.
We can bound the stability term on the RHS of Lemma~\ref{lem:FTRL} as follows.
\begin{lemma}\label{lem:stability_bound_global}
%\label{lem:decompose_global}
%If $\psi_t$ is given by \eqref{eq:defpsiglobal}, then we have
If $\psi_t$ is given by~\eqref{eq:defpsilocal} and $b_t$ is defined by~\eqref{eq:defatbtglobal}, then we have
%\begin{align}\label{eq:decomposeglobal}
%  R_T
%  =
%  O \left(
%  \sum_{t=1}^T
%  \left(
%    \gamma_t
%    +
%    \frac{c_{\calG}^2 b_t}{\beta_t \gamma_t}
%    +
%    (\beta_{t+1} - \beta_t) a_{t+1}
%  \right)
%  +
%  \beta_1 a_1
%  \right)
%  \per
%\end{align}
\begin{align}\label{eq:stability_bound_global} % \label{eq:decomposeglobal}
  \Expect{
    \linner \hat{y}_t, \qt - q_{t+1} \rinner 
    -
    D_t (q_{t+1} , \qt)
  }
  \leq 
  \Expect{
%    \frac{2 \cG^2 b_t}{\beta_t \gamma_t}
    {2 \cG^2 b_t}/{(\beta_t \gamma_t)}
  }
  \per
\end{align}
%with $\{ a_t \}$ and $\{ b_t \}$ defined by \eqref{eq:defatbtglobal}.
\end{lemma}

\begin{remark}
{\rm 
Globally observable PM is a generalization of the weakly observable setting in online learning with feedback graphs~\citep{alon2015online}.
%In the purpose of making the LHS of~\eqref{eq:stability_bound_global} easy to bound and derive a BOBW guarantee, the regularizer in the form of $-H(p) -H(\onemat-p)$ unlike~\eqref{eq:q_for_entropy} is introduced~\citep{ito2022nearly}.
For this online learning problem, the regularizer in the form of $-H(p)-H(\onemat-p)$ rather than~\eqref{eq:q_for_entropy} is introduced in~\citet{ito2022nearly} to make the LHS of~\eqref{eq:stability_bound_global} easy to bound.
% adding the regularizer of the Shannon entropy of $1 - p_a$ leads to a BOBW guarantee~\citep{ito2022nearly}.
However, FTRL with this regularizer requires solving a convex optimization every round.
%This study shows that the stability term can be favorably bounded without the regularization for $1 - p_a$.
This study shows that the LHS of~\eqref{eq:stability_bound_global} can be favorably bounded without the regularization of $-H(\onemat - p)$.
The key to the proof of this lemma is that for any $a' \in [k]$ it holds that
$
  \linner \hat{y}_t, \qt - q_{t+1} \rinner 
  -
  D_t (q_{t+1} , \qt)
  =
  \linner \hat{y}_t - \hat{y}_{t a'} \onemat, \qt - q_{t+1} \rinner 
  -
  D_t (q_{t+1} , \qt)
  \leq
  \beta_t
  \sumak 
  \qta 
  \xi \left(
%    \frac
    ({\hat{y}_{ta} - \hat{y}_{ta'}})/{\beta_t}
  \right)
  ,
%  \com 
$
which enables us to bound the stability term with $b_t$ in~\eqref{eq:defatbtglobal},
leading to the regret upper bound depending on $Q(a^*)$ in Proposition~\ref{prop:global}.
}
\end{remark}
Using the definition of $\beta_t$ and $\gamma_t$ in~\eqref{eq:def_beta_gamma_global} with Lemmas~\ref{lem:FTRL} and~\ref{lem:stability_bound_global}, we can bound the regret as follows.
\begin{proposition}
  \label{prop:global}
  Assume $\beta_t$ and $\gamma_t$ are given by \eqref{eq:def_beta_gamma_global}.
  Then, the regret is bounded as
%  \begin{align}
$
  R_T
  =
  O
  \prn[\Big]{
    \E \Big[
    c_1 
    B_T^{2/3}
    +
    \tilde{c}
    \sqrt{
      c_1^2
      +
      \left( \log \kpi + A_T \right)
      \left( c_1 + B_T^{1/3} \right)
    }
    \Big]
    + \beta_1 \log \kpi
  }
  \com
$
%  \end{align}
  where 
  $A_T = \sum_{t=1}^T a_t$,
  $B_T = \sum_{t=1}^T b_t$,
  and 
%  \begin{align}
$
  \tilde{c} 
  = 
  O \left( 
  \frac{1}{\sqrt{c_1}} \left( \frac{\cG^2 \log T}{c_2} + {c_2} \right) 
  \right)
  = 
  O \prn[\Big]{
  \frac{c_1}{\sqrt{\log(\kpi T)}}
  }
$.
%    \per
%  \end{align}
\end{proposition}
%This lemma is proven by a similar argument as for Proposition 2 of~\citet{ito2022nearly}.
The proof of this lemma is similar to Proposition 2 of~\citet{ito2022nearly}. 
%We are ready to prove Theorem~\ref{thm:globally_obs}.
%Here, we sketch the proof and provides the complete proof in Appendix~\ref{subsec:proof_of_theorem_globally}.
Now we are ready to prove Theorem~\ref{thm:globally_obs}, whose proof is sketched below and completed in Appendix~\ref{subsec:proof_of_theorem_globally}.
%\begin{proofof}{Theorem~\ref{thm:globally_obs}}
\begin{proofsketchof}{Theorem~\ref{thm:globally_obs}}
We first consider the adversarial regime.
%By Proposition~\ref{prop:global} with $A_T \leq T \log \kpi$ and $B_T \leq T$, we have
In the adversarial regime, Proposition~\ref{prop:global} with $A_T \leq T \log \kpi$ and $B_T \leq T$ immediately leads to 
%$
\begin{align}
  \!
  R_T 
%  &
  \!
  =
  O\prn[\Big]{
    c_1 T^{2/3} \!
    + 
    \tilde{c} \sqrt{c_1^2 + (\log \kpi + T \log \kpi) (c_1 + T^{1/3})}
%    +
%    \beta_1 \log \kpi
  }
%  \nn
%  &
  \!
  =
  O\prn[\big]{ 
    \prn[\big]{c_1 + \tilde{c} \sqrt{\log \kpi}}
    T^{2/3}
%    +
%    \sqrt{\frac{\log \kpi}{\log(\kpi T)}} c_1^{3/2} T^{1/2}
%    +
%    \frac{c_1^2}{\sqrt{\log(\kpi T)}}
%    +
%    \beta_1 \log \kpi
  }
  \per
  \label{eq:regret_bound_global_adv_1_abbr}
\end{align}
%$

We next consider the adversarial regime with a self-bounding constraint.
%We can confirm that when ${Q}(a^*) \le c_1^3$ the obtained bound is smaller than the desired bound as follows.
%When $Q(a^*) \le \e$, using Lemma~\ref{lem:entropy_bound} and~\eqref{eq:BT_bound}, we have $A_T \le 2\e\log(\kpi T)$ and $B_T \le 2\e$.
%\todo{fix the analysis around here, since Lemma 8 in the previous manuscript is no more used.}
%Hence, from Proposition~\ref{prop:global}, we have
%\begin{align}
%  R_T
%  =
%  O \prx{
%    c_1 + \tilde{c} \sqrt{c_1^2 + \log(\kpi T) c_1}
%    +
%    \beta_1 \log \kpi
%  }
%  =
%  O \prx{
%    \frac{c_1^2}{\sqrt{\log(\kpi T)}}
%    +
%    \beta_1 \log \kpi
%  }
%  =
%  O \prx{
%    c_1^3
%  }
%  \per
%\end{align}
%When $\e < Q(a^*) \le c_1^3$, % using Lemma~\ref{lem:boundATBT} 
%using Lemma~\ref{lem:entropy_bound} and~\eqref{eq:BT_bound}
%we have $A_T \le 2 c_1^3 \log(\kpi T)$ and $B_T \le 2 c_1^3$.
%Hence, from Proposition~\ref{prop:global}, we have
%\begin{align}
%  R_T
%  &=
%  O \prx{
%    c_1^3 + \tilde{c} \sqrt{c_1^2 + \prx{\log\kpi + c_1^3  \log(\kpi T)} c_1}
%    +
%    \beta_1 \log \kpi
%  }
%  \nn  
%  &=
%  O \prx{
%    {\cG^2 \log(T) \log(\kpi T)}
%  }
%  =
%  O \prx{
%    c_1^3
%  }
%  \per
%\end{align}
Here, we only consider the case of $Q(a^*) > \max\{\e, c_1^3\}$, since otherwise we can easily obtain the desired bound.
Note that $A_T \leq Q(a^*) \log(\kpi T)$ by Lemma~\ref{lem:entropy_bound} with $Q(a^*) \geq \e$ and
$
  B_T
  =
  \sumT \left(1 - \max_{a\in\Pi} \qta\right)
  \leq 
  \sum_{t=1}^T
  \left(
    1 - q_{t,a^*}
  \right)
  =
  Q(a^*)
$.
Then, Proposition~\ref{prop:global} with these inequalities and $Q(a^*) > c_1^3$ gives
\begin{align}
%$
  R_T
%  &
%  =
%  O \prn[\bigg]{
%    \E \left[
%    c_1 {Q}(a^*)^{2/3}
%    +
%    \tilde{c} 
%    \sqrt{
%    c_1^2 
%    + 
%    \prn[\big]{\log \kpi + {Q}(a^*)\log(\kpi T)}
%    \left(c_1 + {Q}(a^*)^{1/3} \right) }
%    \right]
%    +
%    \beta_1 \log \kpi
%  }
%  \nn
%  &
  \leq 
  O \prn[\Big]{
    \E \brk[\Big]{
    c_1 {Q}(a^*)^{2/3}
    +
    \tilde{c} 
    \sqrt{ \log(\kpi T) {Q}(a^*)^{4/3} }
    }
  }
  \leq
  O \prn[\big]{
    \prn[\big]{
    c_1 
    +
    \tilde{c} 
    \sqrt{\log(\kpi T)}
    }
    \bar{Q}(a^*)^{2/3}
  }
  \per 
%  \n 
%  \com
%$
  \label{eq:regret_bound_global_ARSBC_1_abbr}
\end{align}
%where we used $Q(a^*) > c_1^3$ and Jensen's inequality.
%where the first inequality follows from $Q(a^*) > c_1^3$,
%and the second inequality follows from Jensen's inequality.

By~\eqref{eq:regret_bound_global_adv_1_abbr} and~\eqref{eq:regret_bound_global_ARSBC_1_abbr},
there exists
$
%\begin{align}
  \hat{c} = O \prn[\big]{ c_1 + \tilde{c}\sqrt{ \log(\kpi T)} }
%  \per
%\end{align}
$ satisfying
$R_T \leq \hat{c} \, T^{2/3}$ for the adversarial regime
and $R_T \leq \hat{c} \, \bar{Q}(a^*)^{2/3}$ for the adversarial regime with a self-bounding constraint.
Recalling the definitions of $c_1$ and $c_2$, we have
$
%\begin{align}
  \hat{c} 
%  &= 
%  O\prx{
%  \prn[\big]{\cG^2 \log(T) \log(\kpi T)}^{1/3}
%  +
%  \frac{1}{\sqrt{c_1}} \left( \frac{\cG^2 \log T}{c_2} + {c_2} \right) \sqrt{\log(\kpi T)}
%  }
%  \nn
%  &
  =
  O \prx{
  	(\cG^2 \log(T) \log(\kpi T))^{1/3}
  }
%  \com
%  \label{eq:final_c_hat}
%\end{align}
$,
which gives the desired bounds for the adversarial regime.
%For the stochastic regime, 
For the adversarial regime with a self-bounding constraint, 
using $R_T \leq \hat{c} \,\bar{Q}(a^*)^{2/3}$ and Lemma~\ref{lem:selfQ} with $c = 1/2$ for any $\lambda \in (0, 1]$ it holds that
%$
\begin{align}
  R_T
  &
  =
  (1 + \lambda) R_T - \lambda R_T
%  \nn
%  &
  \le
  (1 + \lambda) \hat{c} \cdot \bar{Q}(a^*)^{2/3}
  -
  \lambda \Deltamin \bar{Q}(a^*) / 2
  +
  \lambda C
  \per 
%  \nn 
%  =
%  \left( \frac{2^3 \hat{c}^3}{\Deltamin^2 } \right)^{1/3}
%  \left( \Deltamin \bar{Q}(a^*) \right)^{2/3}
%  -
%  \frac12 \Deltamin \bar{Q}(a^*)
%  \nn
%  &
%  \leq 
%  O \left(
%  {(1 + \lambda)^3 \hat{c}^3}/{ (\lambda^2 \Deltamin^2) }
%  \right)
%  + \lambda C
%%  \nn 
%%  &
%  \leq 
%  O \left(
%    \left(1 + {1}/{\lambda^2} \right) {\hat{c}^3}/{\Deltamin^2}
%  \right) 
%  + \lambda C
%  \com
%  \n 
\end{align}
Taking the worst case of this with respect to $\bar{Q}(a^*)$ and taking $\lambda \in (0,1]$ appropriately gives the desired bound for the adversarial regime with a self-bounding constraint.
%$
%where the second inequality follows from the inequality
%%$a x^{2/3} - (x/2) \le 16 a^3/27$ for $a > 0$.
%$a x^{2/3} - b(x/2) \leq 16a^3/(27b^2)$ for $a, b > 0$,
%and the last equality follows since $\lambda \in (0,1]$.
%Using the value of $\hat{c}$ and then taking $\lambda$ appropriately, we obtain the desired bound for the adversarial regime with a self-bounding constraint.
%\end{proofof}
\end{proofsketchof}
%\end{proof}

%% file: tex_src/appendix.tex
\newpage
\appendix

\section{Intermediate Regimes between Stochastic and Adversarial Regimes}\label{sec:intermediate}
This section details the discussion on intermediate regimes between stochastic and adversarial regimes given in Section~\ref{sec:background}.
This section first defines the \textit{stochastically constrained adversarial regime} in PM, and then shows that the stochastic regime, adversarial regime, stochastically constrained adversarial regime, and stochastic regime with adversarial corruptions are indeed adversarial regimes with a self-bounding constraint defined in Definition~\ref{def:ARSBC}.

The stochastically constrained adversarial regime was initially considered by~\citet{wei2018more} and also discussed in~\cite{zimmert2021tsallis} in the context of the multi-armed bandit problem.
We say that the environment is the stochastically constrained adversarial regime if for any $a \neq a^*$ there exists $\tilde{\Delta}_{a,a^*} > 0$ such that
$\E_{x_t \sim \nu^*}[\calL_{ax_t} - \calL_{a^*x_t} | x_1, \dots, x_{t-1}] \geq \tilde{\Delta}_{a,a^*}$.

Next, we show that the stochastic regime, adversarial regime, stochastically constrained adversarial regime, and stochastic regime with adversarial corruptions are indeed included in the adversarial regime with a self-bounding constraint.
We first consider the stochastic regime.
Indeed, if outcomes $(x_t)_t$ follow a distribution $\nu^*$ independently for $t=1,2,\dots, T$, we have 
$R_T 
= \max_{a^* \in [k]} \E [ \sum_{t=1}^{T} ( \calL_{A_t x_t} - \calL_{a^* x_t} ) ]
= \E [ \sum_{t=1}^{T} \Delta_{A_t} ]
$,
where we define $\Delta \in [0,1]^k$ by $\Delta_a = \E_{x \sim \nu^*}[\calL_{a x} - \calL_{a^* x}]$.
This implies that the stochastic regime is in the adversarial regime with a $(\Delta, 0, T)$ self-bounding constraint.
We next consider the stochastic regime with adversarial corruptions.
In fact, using the definition of the corruption level $C$, we have
\begin{align}
  R_T 
  &=   
  \Expect{\sumT \prx{\lossmat_{\at x_t} - \lossmat_{a^* x_t}}}
  \nn
  &= 
  \Expect{\sumT \prx{\lossmat_{\at x'_t} - \lossmat_{a^* x'_t}}}
  +
  \Expect{\sumT \prx{\lossmat_{\at x_t} - \lossmat_{\at x'_t}}}  
  +
  \Expect{\sumT \prx{\lossmat_{a^* x'_t} - \lossmat_{a^* x_t}}}  
  \nn
  &\geq
  \Expect{\sumT \Delta_{A_t}} - 2C \com 
  \n 
\end{align}
which implies that the stochastic regime with adversarial corruption with corruption levels $C$ is an adversarial regime with a $(\Delta, 2C, T)$ self-bounding constraint.
It is also easy to see that adversarial regimes are the adversarial regime with a $(\Delta, 2T, T)$ self-bounding constraint,
and the stochastically constrained adversarial regime are the adversarial regime with a $(\Delta, 0, T)$ self-bounding constraint by defining $\Delta \in [0,1]^k$ by $\Delta_a = \tilde{\Delta}_{a,a^*}$. % for $a^* = \argmin_{a\in[k]} \tilde{\Delta}_{a,1}$.

\section{Omitted Proofs}\label{sec:proof}
\subsection{Proof of Lemma~\ref{lem:selfQ}}
\begin{proof}
Note that the environment is the adversarial regime with a self-bounding constraint with $\Delta \in [0,1]^k$
such that $\Delta_a \geq \Deltamin$ for all $a \in [k] \setminus \{ a^* \}$.
% Then for any $i^* \in [K]$,
Hence, the regret is then bounded as
\begin{align}
  \nonumber
  R_T
  &
  \geq
  \E \left[
    \sum_{t=1}^T \Delta_{A_t}
  \right]
  -
  C
  =
  \E \left[
    \sum_{t=1}^T \sumak \pta \Delta_a
  \right]
  -
  C
  \\
  &
  \geq
  \E \left[
    \sum_{t=1}^T \sumak c\, \qta \Delta_a
  \right]
  -
  C
  \geq
  c\,\Deltamin \bar{Q}(a^*) 
  -
  C
  \com 
%  \geq
%  \frac{\Deltamin}{2}
%  \bar{Q}
%  -C 
  \n 
\end{align}
where the first inequality follows from Definition~\ref{def:ARSBC},
the equality follows from $A_t \sim p_t$,
the second inequality follows from the definition of $p_t$ given in~\eqref{eq:def_q},
and the last inequality follows from the assumption $\pta \geq c \, \qta$ for all $t \in [T], a \in [k]$ and the definition of $\bar{Q}(a^*)$ given in \eqref{eq:defQ}.
This completes the proof of Lemma~\ref{lem:selfQ}.
\end{proof}

\subsection{Proof of Lemma~\ref{lem:optprime_bound_local}}
Before proving Lemma~\ref{lem:optprime_bound_local}, we review the definition and property of the water transfer operator $W_{\nu}$ introduced by~\citet{lattimore2019information}.
We refer to $\mathscr{T} \subset [k] \times [k]$ representing the edges of a directed tree with vertices $[k]$ as \textit{in-tree} with vertex set $[k]$ and define $\calE = \{(a,b) \in [k] \times [k] : a \mbox{ and } b \mbox{ are neighbors} \}$.
\begin{lemma}[\citealp{lattimore2019information}]\label{lem:wto}
Assume that partial monitoring game $\calG$ is non-degenerate and locally observable and let $\nu \in \calP_d$. 
Then there exists a function $W_\nu : \calP_k \to \calP_k$ such that the following hold for all $q \in \calP_k$:
(a) $\left(W_\nu(q) - q\right)^\top \calL \nu \leq 0$;\
(b) $W_\nu(q)_a \geq q_a/k$ for all $a \in [k]$;\ and
(c) there exists an in-tree $\mathscr{T} \subset \calE$ over $[k]$ such that $W_\nu(q)_a \leq W_\nu(q)_b$ for all $(a, b) \in \mathscr{T}$.
\end{lemma}

Using this, we prove the generalized version of Proposition 8 of~\citet{lattimore20exploration},
where the proof follows a quite similar argument as their proof therein.
\begin{proofof}{Lemma~\ref{lem:optprime_bound_local}}
We define the set of functions that satisfy~\eqref{eq:Gdiff_Ldiff} by
\begin{align}
  \calH_{\circ}
  =
  \left\{
    G : 
    (e_b - e_c)^\top \sumak G(a, \Phi_{ax})
    = 
    \calL_{bx} - \calL_{cx} \mbox{ for all } b, c \in \Pi \mbox{ and } x \in [d]
  \right\}
  \per 
  \n 
\end{align}
Take any $q \in \calP_k$.
By Sion's minimax theorem, we have
\begin{align}
  \opt'_q(\eta)
  &\leq 
  \min_{G \in \calH^{\circ},\, p \in \calP'_k(q)} 
  \max_{\nu \in \calP_d} 
  \left[
  \frac{1}{\eta} (p - q)^\top \calL \nu
  + 
  \frac{1}{\eta^2} \sum_{x=1}^d 
  \nu_x 
  \sumak 
  p_a 
  \innerprod{q}{\xi\left( \frac{\eta G(a, \fbmat_{ax})}{p_a} \right)}
  \right]
  \nn 
  &= 
  \max_{\nu \in \calP_d} 
  \min_{G \in \calH^{\circ},\, p \in \calP'_k(q)} 
  \left[
  \frac{1}{\eta} (p - q)^\top \calL \nu 
  + 
  \frac{1}{\eta^2} \sum_{x=1}^d \nu_x 
  \sumak 
  p_a 
  \innerprod{q}{\xi\left( \frac{\eta G(a, \fbmat_{ax})}{p_a} \right)}
  \right] 
  \com 
  \n 
\end{align}
where the first inequality follows since we added the constraint that $G \in \calH^{\circ}$, which makes the bias term zero.
Take any $\nu \in \calP_d$ and let $\mathscr{T}$ be the in-tree over $[k]$.
Using these variables, we define the action selection probability vector $p \in \calP'_k(q)$ by 
\begin{align}
  p = (1 - \gamma) u + \frac{\gamma}{k} \onemat
  \com
  \quad 
  \mbox{where}
  \quad 
  u = W_\nu(q)
  \com 
  \quad
  \mbox{and}
  \quad
  \gamma 
  = 
  \frac{\eta m k^2}{2}
  \per 
  \n 
\end{align}
Here, $W_\nu : \calP_k \to \calP_k$ is the water operator.
It is worth noting that from the assumption that $\eta \leq 1/(mk^2)$, 
we have $\gamma \leq 1/2$ and $p_a \geq u_a/2 = W_\nu(q)_a /2 \geq q_a/(2k)$,
where the last inequality follows from Part (b) of Lemma~\ref{lem:wto},
and this indeed implies $p \in \calP'_k(q)$.

We take $G \in \calH^{\circ}$ defined in~\eqref{eq:G_0},
where we recall that 
$
G(a, \sigma)_b = \sum_{e \in \mathrm{path}_\calT(b)} w_e(a, \sigma) . 
$
By Lemma 20 of~\citet{lattimore20exploration} and the assumption that $\calG$ is non-degenerate, $w_e$ can be chosen so that $\nrm{w_e}_\infty \leq m/2$.
Since paths in $\calT$ have length at most $k$, we have $\nrm{G}_\infty \leq km/2$.
From the above definitions, for any $x \in [d]$ we have
\begin{align}
  \frac{\eta G(a, \Phi_{ax})}{p_a} \geq -\frac{\eta m k^2}{2\gamma} = -1 
  \per
  \n 
\end{align}
Hence, using
%~\eqref{eq:xi_bound},
Parts (b) and (c) of Lemma~\ref{lem:wto}, we have
\begin{align}
  \frac{1}{\eta^2} 
\sumak 
  p_a 
  \innerprod{q}{\xi\left( \frac{\eta G(a, \fbmat_{ax})}{p_a} \right)}
  &\leq
%  \sumak \frac{\nrm{G(a, \Phi_{ax})}^2_{\diag(q)}}{p_a} \nn 
  \sumak \frac{1}{p_a} \sum_{b=1}^k q_b \, (G(a, \Phi_{ax})_b)^2 \nn 
  &\leq 
  2\sumak \frac{1}{u_a} \sum_{b=1}^k q_b \, (G(a, \Phi_{ax})_b)^2  \nn 
  &= 
  2\sum_{b=1}^k \sumak \frac{q_b}{u_a} \left(\sum_{e \in \mathrm{path}_{\calT}(b)} w_e(a, \Phi_{ax}) \right)^2 \nn 
  &\leq
  \frac{m^2}{2}\sum_{b=1}^k \sumak \frac{q_b}{u_a} \left(\sum_{e  \in \mathrm{path}_{\calT}(b)} \ind{a \in e}\right)^2 \nn 
  &\leq
  2 m^2 k^3 \com 
  \n 
\end{align}
where the first inequality follows from % from~\eqref{eq:xi_bound},
%Note that $\xi$ satisfies 
\begin{align}\label{eq:xi_bound}
  \xi(x) = \exp(-x) + x - 1 \leq x^2 \mbox{ for } x \geq -1 
  \com
\end{align}
the second inequality follows since $p_a \geq u_a/2$,
the third inequality follows since $\nrm{w_e}_\infty \leq m/2$,
and 
the last inequality follows from Part (b) of Lemma~\ref{lem:wto} to implying that $q_b \leq k u_b$ and Part (c) implying that $u_a \geq u_b$ for $a \in \mathrm{path}_{\calT}(b)$.
Finally,
\begin{align}
  \frac{1}{\eta} (p - q)^\top \calL \nu 
  &= 
  \frac{1}{\eta}(u - q)^\top \calL \nu 
  + 
  \frac{\gamma}{\eta} \left(\frac1k \onemat - u \right)^\top \calL \nu
  \leq 
  \frac{\gamma}{\eta} \left(\frac1k \onemat - u \right)^\top \calL \nu
  \leq  
  m k^2 
  \com 
  \n 
\end{align}
where the first inequality follows from Part (a) of Lemma~\ref{lem:wto}.
Summing up the above arguments, we have $\opt'_q(\eta) \leq 3 m^2 k^3$,
which completes the proof of Lemma~\ref{lem:optprime_bound_local}.
\end{proofof}

\subsection{Proof of Lemma~\ref{lem:ftrl_bound_with_opt}}
We first analyze the stability term in~\eqref{eq:lemFTRL} for $\psi_t$ defined in \eqref{eq:defpsilocal}.
\begin{lemma}
  \label{lem:boundBregLocal}
  If $\psi_t$ is given by \eqref{eq:defpsilocal},
  it holds for any $\ell \in \R^k$ and
  $p, q \in \calP_k$ that
  \begin{align}
  \linner \ell, p - q \rinner 
  &
  -
  D_t ( q , p )
  \leq
  \beta_t
  \sumak
  \pa
  \xi \left(
    \frac{\ell_a}{\beta_t}
  \right)
  \com
  \n 
  \end{align}
  where we recall that $\xi(x) = \exp(-x) + x - 1$.
\end{lemma}
\begin{proof}
  For any $x, y \in (0, 1)$,
  we let $\dent(y, x) \ge 0$ be the Bregman divergence over $(0, 1)$ induced by $\psi(x) = x \log x$, \ie
  \begin{align}
%    \label{eq:defd1}
    \dent(y, x)
    &
    =
    y \log y - x \log x - ( \log x + 1 )(y - x) 
    =
    y \log \frac{y}{x}
    +
    x - y
    \per
    \n 
  \end{align}
  Using this, the Bregman divergence induced by $\psi_t(p) = ({1}/{\eta_t}) \sumak p_a \log(p_a) = \beta_t \sumak p_a \log(p_a)$ in~\eqref{eq:defpsilocal} can be written as
  \begin{align}
    D_t(q, p)
    &
    =
    \psi_t(p) - \psi_t(q) - \innerprod{\nabla \psi_t(q)}{p - q}
    =
    \beta_{t}
    \sumak
      \dent(\qa , \pa)
    \per
    \n 
  \end{align}
  From this, we have
  \begin{align}
%    &
    \linner \ell, p - q \rinner
    - D_t(q, p)
    \le
    \sumak
    \left(
      \ell_a (\pa - \qa)
      -
      \beta_{t} 
      \dent(\qa, \pa) 
    \right)
    \per
    \label{eq:stab_decompose}
  \end{align}
  We show
  \begin{align}
    \ell_a (\pa - \qa)
    -
    \beta_{t}
    \dent(\qa, \pa)
    \leq
    \beta_t
    \pa
    \xi \left( 
      \frac{\ell_a}{ \beta_t }
    \right)
    \per
    \label{eq:stab_bound_sub}
  \end{align}
  As 
  $
  \ell_a ( \pa - \qa )
  -
  \beta_{t}
  \dent(\qa, \pa)
  $
  is concave in $q$,
  its maximum subject to $q \in \R$ is attained when
  the derivative of it is equal to zero, \ie
  \begin{align}
    \frac{\partial}{\partial \qa}
    \left(
    \ell_a ( \pa - \qa )
    -
    \beta_{t}
    \dent(\qa, \pa)
    \right)
    =
    -
    \ell_a
    -
    \beta_t \left(
      \log \qa
      -
      \log \pa 
    \right)
    =
    0
    \n 
    \per
  \end{align}
  This implies that the maximum is attained when
  $\qa = q^*_{a} \coloneqq \pa \exp\left( - {\ell_a}/{\beta_t} \right)$.
  Hence,
  we obtain~\eqref{eq:stab_bound_sub} by
  \begin{align}
    &
    \ell_a ( \pa - \qa )
    -
    \beta_{t}
    \dent(\qa, \pa)  
    \le
    \ell_a ( \pa - q^*_a )
    -
    \beta_{t}
    \dent(q^*_a, \pa)  
    \nn
    &
    =
    \ell_a ( p_a - q^*_a )
    -
    \beta_t
    \left(
    q^*_a \log q^*_a
    -
    p_a \log p_a
    -
    (\log p_a + 1)
    (q^*_a - p_a)
    \right)
    \nn
    &
    =
    \ell_a p_a 
    -
    \beta_t
    \left(
      q^*_a
      \log p_a
    -
    p_a \log p_a
    -
    (\log p_a + 1)
    (q^*_a - p_a)
    \right)
    \nn
    &
    =
    \ell_a p_a 
    +
    \beta_t
    (q^*_a - p_a)
    =
    \beta_t
    p_a \left( \exp\left( - \frac{\ell_a}{\beta_t} \right) + \frac{ \ell_a}{\beta_t} - 1 \right)
    \nn
    &
    =
    \beta_t
    p_a
    \xi \left(
      \frac{\ell_a}{\beta_t}
    \right)
    \com
    \n 
  \end{align}
where the second equality follows from $\log q^*_a = \log p_a - {\ell_a}/{\beta_t}$,
and the fourth equality follows from $q^*_a = p_a \exp \left( - {\ell_a}/{\beta_t} \right)$.
Combining~\eqref{eq:stab_decompose} and~\eqref{eq:stab_bound_sub} completes the proof.  
\end{proof}

\begin{proofof}{Lemma~\ref{lem:ftrl_bound_with_opt}}
Let $a^* = \argmin_{a \in [k]} \E \brk[\big]{\sumT \lossmat_{a x_t}} \in \Pi$ be the optimal action in hindsight.
% , where ties are broken so that $a^* \in \Pi$.
We have
\begin{align}\label{eq:decompose_local}
    R_T
    &= 
    \Expect{\sumT (\lossmat_{A_t x_t} - \lossmat_{a^* x_t})}  
%    \nn
%    &
    = 
    \Expect{\sumT \sum_{b=1}^k \ptb (\lossmat_{b x_t} - \lossmat_{a^* x_t})}  \nn
    &= 
    \Expect{
    \sumT \sum_{b=1}^k (\ptb - \qtb) (\lossmat_{b x_t} - \lossmat_{a^* x_t})
    +
    \sumT \sum_{b=1}^k \qtb (\lossmat_{b x_t} - \lossmat_{a^* x_t})
%    +
%    \sumT \sum_{b=1}^k \qtb (\lossmat_{b x_t} - \lossmat_{a^* x_t})
    }
    \per
\end{align}
The first term in~\eqref{eq:decompose_local} is equal to
$\E \brk[\big]{ \sumT (\pt - \qt)^\top \calL e_{x_t} }$.
The second term in~\eqref{eq:decompose_local} can be bounded as
\begin{align}\label{eq:bias_qregret}
    \Expect{
    \sum_{b=1}^k \qtb (\lossmat_{b x_t} - \lossmat_{a^* x_t})
    }
    &=
    \Expect{
    \sum_{b=1}^k \qt^\top \lossmat e_{x_t} - \lossmat_{a^* x_t}
    }  
    \nn
    &=
    \Expect{
    \sum_{b=1}^k 
    \qt^\top \calL e_{x_t}
    - 
    \qt^\top \sumak G_t (a, \fbmat_{a x_t})
    +
    \sumak G_t (a, \fbmat_{a x_t})_{a^*}
    - 
    \lossmat_{a^* x_t}
    }
    \nn
    &\quad
    +
    \Expect{
    \qt^\top \sumak G_t (a, \fbmat_{a x_t})
    -
    \sumak G_t (a, \fbmat_{a x_t})_{a^*}
    }
    \nn
    &\le
    \Expect{
    \bias_{\qt}(G; x_t)
    }
    +
    \Expect{
    \qt^\top \hat{y}_t - \hat{y}_{ta^*}
    }
    \com
\end{align}
where in the last inequality we used the definition in~\eqref{eq:defbias} and Lemma~\ref{lem:Gdiff_Ldiff} with $a^* \in \Pi$ and $\qta = 0$ for $a\not\in\Pi$.
The sum over $t \in [T]$ of the last term in~\eqref{eq:bias_qregret} can be bounded using \eqref{eq:lemFTRL} and the definition of the regularizer~\eqref{eq:defpsilocal} as
\begin{align}\label{eq:ftrl_local}
  &
  \Expect{\sumT \sum_{b=1}^k \qtb (\hat{y}_{tb} - \hat{y}_{ta^*})}
  \nn 
  &\leq
  \Expect{
  \sumT
  \prx{
      \frac{1}{\eta_{t+1}}
      -
      \frac{1}{\eta_t}
  }
  H(q_{t+1})
  +
  \frac{H(q_1)}{\eta_1}
  + 
  \sumT
  \prn[\Big]{
  \innerprod{\qt - q_{t+1}}{\hat{y}_t}
  -
  D_t(q_{t+1}, \qt)
  }
  }
  \nn
  &\leq 
  \Expect{
  \sumT
  \prx{
      \frac{1}{\eta_{t+1}}
      -
      \frac{1}{\eta_t}
  }
  H(q_{t+1})
  +
  \frac{H(q_1)}{\eta_1}
  + 
  \sumT \frac{\innerprod{\qt}{\xi(\eta_t \hat{y}_t)}}{\eta_t} 
  }
  \com
\end{align}
where in the last inequality we used the following inequality obtained by Lemma~\ref{lem:boundBregLocal}:
\begin{align}% \label{eq:stab_bound_instance_local}
  \linner \hat{y}_t, \qt - q_{t+1} \rinner 
  &
  -
  D_t (q_{t+1} , \qt)
  \leq
  \beta_t
  \sumak
  \qta 
  \xi \left(
    \frac{\hat{y}_{ta}}{\beta_t}
  \right)
  =
  \frac{\innerprod{\qt}{\xi(\eta_t \hat y_t)}}{\eta_t}
  \per
  \n 
\end{align}
Using the definition of the optimization problem~\eqref{eq:optprime} and $V'_t = \max\{0, \opt'_{\qt}(\eta_t)\}$, we have
\begin{align}\label{eq:equality_local}
  (\pt - \qt)^\top \calL e_{x_t}
  +
  \bias_{\qt}(G; x_t)
  +
  \frac{\innerprod{\qt}{\xi(\eta_t \hat{y}_t)}}{\eta_t} 
  \le 
  \eta_t V'_t
  \per
\end{align}
Summing up the arguments in \eqref{eq:decompose_local}, \eqref{eq:bias_qregret}, \eqref{eq:ftrl_local}, and \eqref{eq:equality_local}, we have
\begin{align}
  R_T
  &\leq 
  \Expect{
  \sumT
  \prx{
      \frac{1}{\eta_{t+1}}
      -
      \frac{1}{\eta_t}
  }
  H(q_{t+1})
  +
  \frac{H(q_1)}{\eta_1}
  + 
  \sumT \eta_t V'_t
  }
  \com 
  \n 
\end{align}
which completes the proof.
\end{proofof}

\subsection{Proof of Lemma~\ref{lem:entropy_bound}}
\begin{proof} % {Lemma~\ref{lem:entropy_bound}}
For any $q \in \calP(\Pi)$ and $a^* \in \Pi$,
we have
\begin{align}
  H(p)
  &
  =
  \sum_{a\in\Pi} \qa \log \frac{1}{\qa}
  =
  \sum_{a\in\Pi \setminus \{ a^* \}} \qa \log \frac{1}{\qa}
  +
  q_{a^*} \log \left( 1 + \frac{1-q_{a^*}}{q_{a^*}} \right)
  \nonumber
  \\
  &
  \le 
  (\kpi - 1) \sum_{a\in\Pi \setminus \{ a^* \}} \frac{1}{\kpi - 1} \, \qa \log \frac{1}{\qa}
  +
  q_{a^*} \frac{1-q_{a^*}}{q_{a^*}}
  \nn
  &
  \leq
  (\kpi-1)
  \cdot
  \frac{\sum_{a\in\Pi \setminus \{ a^* \}} \qa}{\kpi-1} \log 
  \frac{\kpi-1}{\sum_{a\in\Pi \setminus \{ a^* \}} \qa}
  +
  q_{a^*} \frac{1-q_{a^*}}{q_{a^*}}
  \nonumber
  \\
  &
  =
  (1 - q_{a^*})
  \left(
    \log \frac{\kpi-1}{1 - q_{a^*}}
    +
    1
  \right)
  \leq 
  (1 - q_{a^*})
%  \left(
    \log \frac{\e \kpi}{1 - q_{a^*}}
%  \right)
  \com
\label{eq:boundH}
\end{align}
where the first inequality follows from $\log ( 1 + x ) \leq x$ for $x \geq 0$,
the last inequality follows from Jensen's inequality, 
and the last equality follows from $\sum_{a\in\Pi} \qa = 1$.
Using~\eqref{eq:boundH}, for any $a^* \in [k]$ we have
\begin{align}
  \sum_{t=1}^T a_t
  = 
  \sum_{t=1}^T H( q_t )
  &
  \leq
  \sum_{t=1}^T
  (1 - q_{ta^*})
%  \left(
    \log \frac{\e \kpi}{1 - q_{ta^*}}
%  \right)
  \nn
  &=
  T \sum_{t=1}^T \frac1T
  (1 - q_{t,a^*}) \log \frac{\e \kpi}{1 - q_{t,a^*}}  
  \nn
  &\le 
  T 
  \prx{\sum_{t=1}^T \frac1T (1 - q_{t,a^*})}
  \log \frac{\e \kpi}{\sum_{t=1}^T \frac1T \prx{1 - q_{t,a^*}}}  
  \nn
  &=
  T \frac{Q(a^*)}{T} \log \frac{\e \kpi T}{Q(a^*)}
%  \nn
%  &
%  \leq
%  Q(a^*)
%  \left(
%    \log \frac{(k-1)T}{Q(a^*)}
%    +
%    1
%  \right)
  =
  Q(a^*)
  \left(
    \log \frac{\e \kpi T}{Q(a^*)}
  \right)
  \com 
  \n 
\end{align}
where in the second inequality we used Jensen's inequality since $f(x) = x \log(1/x)$ is concave,
and in the third inequality we used the definition of $Q(a^*)$ in~\eqref{eq:defQ}.
%\end{proofof}
\end{proof}

\subsection{Proof of Theorem~\ref{thm:locally_obs}}\label{subsec:proof_of_theorem_locally}
%\begin{proofof}{Theorem~\ref{thm:locally_obs}}
\begin{proof}
We prove this theorem by bounding the RHS of Lemma~\ref{lem:ftrl_bound_with_opt}.
\paragraph{[Bounding the penalty term]}
Let $t_0 = \min\{t \in [T] : \beta'_t \geq B\}$.
Then, the definition of the learning rate~\eqref{eq:defeta} gives that
\allowdisplaybreaks
\begin{align}
  &
  \sumT
  \prx{
      \frac{1}{\eta_{t+1}}
      -
      \frac{1}{\eta_t}
  }
  H(q_{t+1})
  =  
  \sumT
  \prx{
      \beta_{t+1}
      -
      \beta_t
  }
  H(q_{t+1})
  \nn
  &=
  \sum_{t=1}^{t_0 - 2}
  \prx{
      \beta_{t+1}
      -
      \beta_t
  }
  H(q_{t+1})
  +
  \prx{
      \beta_{t_0}
      -
      \beta_{t_0 - 1}
  }
  H(q_{t+1})
  +
  \sum_{t=t_0}^{T}
  \prx{
      \beta_{t+1}
      -
      \beta_t
  }
  H(q_{t+1})
  \nn
  &
  \le 
  0
  +
  \prx{
      \beta'_{t_0}
      -
      \beta'_{t_0 - 1}
  }
  H(q_{t+1})
  +
  \sum_{t=t_0}^{T}
  \prx{
      \beta'_{t+1}
      -
      \beta'_t
  }
  H(q_{t+1})
  \nn 
  &\leq
  \sumT
  \prx{
      \beta'_{t+1}
      -
      \beta'_t
  }
  H(q_{t+1})
  \com
  \n 
\end{align}
where
in the first inequality we used 
the fact that $\beta'_{t+1}$ is non-decreasing,
$\beta_{t+1} = \beta_t$ for $t \le t_0 - 1$,
$\beta'_t \le \beta_t$,
and 
$\beta'_t = \beta_t$ for $t \ge t_0$.
Using this inequality, we have
\allowdisplaybreaks
\begin{align}\label{eq:penalty_etadiffsum_bound}
  &
  \sumT
  \prx{
    \frac{1}{\eta_{t+1}}
    -
    \frac{1}{\eta_t}
  }
  H(q_{t+1})
%  &
  \leq  
  \sumT
  \prx{
    \beta'_{t+1}
    -
    \beta'_t
  }
  H(q_{t+1})
  \nn
  &
  =
  % {c_1 }
  \sum_{t=1}^T
  \frac{ c_1 }{
    \sqrt{1 + (\log \kpi)^{-1} \sum_{s=1}^{t} H(q_s) }
  }
  \cdot
  H(q_{t+1})
  \nn
  &=
  {2 c_1 \sqrt{\log \kpi}}
  \sum_{t=1}^T
  \frac{ H(q_{t+1}) }{
    \sqrt{\log \kpi + \sum_{s=1}^{t} H(q_s) } + 
    \sqrt{\log \kpi + \sum_{s=1}^{t} H(q_s) } 
  }
  \nn
  &
  \leq
  {2 c_1 \sqrt{\log \kpi}}
  \sum_{t=1}^T
  \frac{ H(q_{t+1}) }{
    \sqrt{\sum_{s=1}^{t+1} H(q_s) } + 
    \sqrt{\sum_{s=1}^{t} H(q_s) } 
  }
  \nn
  &
  =
  {2 c_1 \sqrt{\log \kpi}}
  \sum_{t=1}^T
  \left(
    \sqrt{\sum_{s=1}^{t+1} H(q_s) } -
    \sqrt{\sum_{s=1}^{t} H(q_s) } 
  \right)
  % \\
  % &
  \nn
  &
  =
  {2 c_1 \sqrt{\log \kpi}}
  \left(
    \sqrt{\sum_{s=1}^{T+1} H(q_s) } -
    \sqrt{H(q_{1})} 
  \right)
  \nn
  &
  \leq 
  {2 c_1 \sqrt{\log \kpi}}
  \left(
    \sqrt{\sum_{s=2}^{T+1} H(q_s) } 
%      -
%      \sqrt{H(q_{1})} 
  \right)
%  \nn
%  &
  \le
  {2 c_1 \sqrt{\log \kpi}}
  \sqrt{\sum_{t=1}^{T} H(q_{t})}
  \com
\end{align}
where 
the second inequality follows from $0 \le H(q_{t+1}) \le \log \kpi$,
the third inequality follows from the inequality $\sqrt{a+b} - \sqrt{b} \le \sqrt{a}$ that holds for $a, b \ge 0$,
and the last inequality follows since $H(q_{T+1}) \leq H(q_{1})$.
% ==
\paragraph{[Bounding the sum of the transformation and part of stability term]}
%Next, we bound $\sumT \eta_t$.
Using the definition of $\beta'_t$ in~\eqref{eq:defeta}, we can bound $\beta'_t$ as
\begin{align}
  \beta'_t
  =
  c_1
  +
  \sum_{u=1}^{t-1}
  \frac{ c_1} {\sqrt{ 1 + (\log \kpi)^{-1} \sum_{s=1}^{u} H(q_s) }}
  \geq
  \frac{ c_1 t } {\sqrt{ 1 + (\log \kpi)^{-1} \sum_{s=1}^{t} H(q_s) } }  
  \per
  \n 
\end{align}
Using this inequality, we have
\begin{align}
  \sumT \eta_t 
  \leq 
  \sum_{t=1}^T
  \frac{1}{\beta'_t}
  \leq
  \sum_{t=1}^T
  \frac{1}{ c_1 t } 
  \sqrt{ 1 + \frac{1}{\log \kpi} \sum_{s=1}^{t} H(q_s) } 
%    &
  \leq
  \frac{1 + \log T}{ c_1 } 
  \sqrt{ 1 + \frac{1}{\log \kpi} \sum_{t=1}^{T} H(\qt) } 
  \per
  \label{eq:sum_eta_bound}
\end{align}
Further, we have
\begin{align}\label{eq:sum_etaV_bound}
  \sumT \eta_t V'_t
  \leq 
  \max_{s \in [T]} V'_s \sumT \eta_t
  =
  \prx{
  \max_{s \in [T]} 
  \max\brx{0,\, \opt'_*(\eta_s)}
  }
  \sumT \eta_t
  \le
  3 m^2 k^3
  \sumT \eta_t
  \com
\end{align}
where in the last inequality we used Lemma~\ref{lem:optprime_bound_local} with $\eta_t \le 1/(2 m k^2)$.

\paragraph{[Summing up the above arguments with a self-bounding technique]}
By bounding the RHS of Lemma~\ref{lem:ftrl_bound_with_opt} using \eqref{eq:penalty_etadiffsum_bound}, \eqref{eq:sum_eta_bound}, and \eqref{eq:sum_etaV_bound}, we have
\begin{align}\label{eq:entropy_dep_bound_local}
  R_T 
  &\le
  3 m^2 k^3 \Expect{\frac{1 + \log T}{c_1} \sqrt{1 + (\log \kpi)^{-1} \sumT H(q_t)}}
  +
  2 c_1 \sqrt{\log \kpi} \, \Expect{\sqrt{\sumT H(q_t)}}
  +
  \frac{\log \kpi}{\eta_1}
  \nn
  &=
  O\prx{
    m k^{3/2} \sqrt{\log (T) \sumT H(q_t)} 
    +
    m k^{3/2} \sqrt{\log (T) \log \kpi}
  }
  +
  2mk^2 \log \kpi
  \com
\end{align}
where we set
%\begin{align}
$
  c_1 = \Theta\prx{
  m k^{3/2} \sqrt{\frac{\log T}{\log \kpi}}
  }
$.
%  \per
%\end{align}

The desired bound is obtained for the adversarial regime, since $\sumT H(q_t) \le T \log \kpi$.
We consider the stochastic regime in the following.
If $Q(a^*) \leq \e$, Lemma~\ref{lem:entropy_bound} implies $\sumT H(q_t) \leq \e \log(\kpi T)$ since $\kpi T \geq \e$, and otherwise we have $\sumT H(q_t) \leq Q(a^*) \log(\kpi T)$.
In the former case, we can trivially obtain the desired bound immediately from~\eqref{eq:entropy_dep_bound_local}.
For the latter case, using the inequality $\sumT H(q_t) \le Q(a^*) \log(\kpi T)$,~\eqref{eq:sum_etaV_bound}, and Lemma~\ref{lem:selfQ} with $c = 1/(2k)$, % and Lemma~\ref{lem:entropy_bound}, 
we have for any $\lambda > 0$ that 
\begin{align}
  &
  R_T 
  = 
  (1 + \lambda) R_T - \lambda R_T
  \leq
  \Expect{
  (1 + \lambda)
  O\prx{ m k^{3/2} \sqrt{\log(T) \log(\kpi T) Q(a^*)} }
  -
  \frac{\lambda \Deltamin}{2k} Q(a^*)
  }
  + \lambda C
  \nn
  &\leq 
  % \frac{4 m^2 k^3 \log(T) \cdot \log(\kpi T)}{2 \Deltamin}
  O\prx{
    \frac{(1 + \lambda)^2 m^2 k^4 \log(T) \log(\kpi T)}{\lambda\Deltamin}
  }
  + \lambda C
  \nn 
  &=
  O\prx{
    \frac{m^2 k^4 \log(T) \log(\kpi T)}{\Deltamin}
    +
    \lambda \left( \frac{m^2 k^4 \log(T) \log(\kpi T)}{\Deltamin} + C \right)
    +
    \frac1\lambda \frac{m^2 k^4 \log(T) \log(\kpi T)}{\Deltamin}
  }
  \com
  \label{eq:regert_selfbound_local}
\end{align}
where the second inequality follows from
$a \sqrt{x} - bx/2 \le a^2/(2b)$, which holds for any $a, b, x \geq 0$.
Taking
\begin{align}
%$
  \lambda
  = 
  O\left(
  \sqrt{
    m^2 k^4 \log(T) \log(\kpi T)
    \Big/\Big.
    \left( \frac{m^2 k^4 \log(T) \log(\kpi T)}{\Deltamin} + C \right)
  }
  \right)
  \n 
%$
\end{align}
completes the proof.
%\end{proofof}
\end{proof}

\subsection{Proof of Lemma~\ref{lem:FTRL}}
\begin{proof}
Let $a^* = \argmin_{a \in [k]} \E \brk[\big]{\sumT \lossmat_{a x_t}}$ be the optimal action in hindsight, where ties are broken so that $a^* \in \Pi$.
Note that since action $a$ with $\dim(\calC_a) < d - 1$ cannot be uniquely optimal, one can see that we can take action $b \in \Pi$ instead of such $a$ with the same loss.
We have
\begin{align}
  R_T
  &= 
  \Expect{\sumT \prx{\lossmat_{\at, x_t} - \lossmat_{a^*,x_t}}}
  =
  \Expect{\sumT \innerprod{\pt - e_{a^*}}{\lossmat e_{x_t}}}
  \nn
  &=
  \Expect{
  \sumT \innerprod{\qt - e_{a^*}}{\lossmat e_{x_t}}
  + 
  \sumT \gamma_t \innerprod{\frac1k \onemat - \qt}{\lossmat e_{x_t}}
  }
  \nn
  &
  \le
  \Expect{
  \sumT \innerprod{\qt - e_{a^*}}{\lossmat e_{x_t}}
  + 
  \sumT \gamma_t 
  }
  % \nn
  % &
  =
  \Expect{
  \sumT \sumak \qta \prx{\lossmat_{a x_t} - \lossmat_{a^* x_t}}
  +
  \sumT \gamma_t 
  }
  \nn
  &
  =
  \Expect{
  \sumT \sumak \qta \prx{\hat y_{t a} - \hat y_{t a^*}}
  +
  \sumT \gamma_t 
  }
  =
  \Expect{
  \sumT \innerprod{q_t - e_{a^*}}{\hat y_t}
  +
  \sumT \gamma_t 
  }
  \com
  \n 
\end{align}
where the inequality follows from the boundedness of $\lossmat$, 
the fourth equality follows since $a^* \in \Pi$, $\qta = 0$ for $a\not\in\Pi$, and Lemma~\ref{lem:Gdiff_Ldiff},
and the fifth equality follows from the definitions of $\hat{y}$ and $\qta = 0$ for $a \not\in \Pi$.
Combining the above inequality and \eqref{eq:lemFTRL} completes the proof.
\end{proof}

%\subsection{Proof of Lemma~\ref{lem:decompose_global}}
\subsection{Proof of Lemma~\ref{lem:stability_bound_global}}

\begin{proof} % {Lemma~\ref{lem:decompose_global}}
We first bound the stability term.
Using Lemma~\ref{lem:boundBregLocal},
for any $a'\in\calA$ it holds that
\begin{align} % \label{eq:stab_bound_instance_global}
  \linner \hat{y}_t, \qt - q_{t+1} \rinner 
  -
  D_t (q_{t+1} , \qt)
  &=
  \linner \hat{y}_t - \hat{y}_{t a'} \onemat, \qt - q_{t+1} \rinner 
  -
  D_t (q_{t+1} , \qt)
  \nn 
  &\leq
  \beta_t
  \sumak 
%  \min \left\{
  \qta 
  \xi \left(
    \frac{\hat{y}_{ta} - \hat{y}_{ta'}}{\beta_t}
  \right)
%  ,
%  (1 - \qta)
%  \xi \left(
%    -
%    \frac{\hat{y}_{ta}}{\beta_t}
%  \right)
%  \right\}
  \per
  \n 
\end{align}
We evaluate the RHS of this inequality.
As we define $p_t$ by \eqref{eq:p_global},
we have $\pta \geq \gamma_t / k$ for any $a \in [k]$.
We first show that $|({\hat{y}_{ta} - \hat{y}_{ta'}}) / {\beta_t}| \le 1$ for all $a, a' \in [k]$.
Let $\tau = \nrm{G}_\infty$.
Recall that $c_{\calG} = \max\{1, k \tau\}$.
Then we have
\begin{align}
  \frac{\hat y_t}{{\beta}_t}
  =
  \frac{G(a, \Phi_{ax})}{{\beta}_t \, \ptAt}
  \ge
  - \frac{\tau}{{\beta}_t \, \ptAt} \onemat
  \ge
  - \frac12 \onemat
  \com
  \n 
\end{align}
where the inequalities here are element-wise,
the first inequality follows from the definition of $\tau$,
and in the last inequality we used
$\pta \ge {\gamma_t}/{k} \ge {c_{\calG}}/{(2\beta_t k)} \ge {\tau}/(2\beta_t)$ for all $a\in[k]$.
In a similar manner we have
\begin{align}
    \frac{\hat y_t}{{\beta}_t}
    =
    \frac{G(a, \Phi_{ax})}{{\beta}_t \,\ptAt}
    \le
    \frac{\tau}{{\beta}_t \, \ptAt} \onemat
    \le
    \frac12 \onemat
    \per
    \n 
\end{align}
These arguments conclude that 
$|({\hat{y}_{ta} - \hat{y}_{ta'}}) / {\beta_t}| \leq |{\hat{y}_{ta}}/{\beta_t}| + |\hat{y}_{ta'} / {\beta_t}| \leq 1$ 
for all $a, a' \in [k]$.
Hence, we have % for any $a' \in [k]$ that
\begin{align}
  \linner \hat{y}_t, \qt - q_{t+1} \rinner 
  -
  D_t (q_{t+1} , \qt)
%  &=
%  \linner \hat{y}_t - \hat{y}_{ta'} \onemat, \qt - q_{t+1} \rinner 
%  -
%  D_t (q_{t+1} , \qt)
%  \nn 
  &\leq
  \min_{a' \in [k]}
  \beta_t
  \sumak 
  \qta 
  \left(
    \frac{\hat{y}_{ta} - \hat{y}_{ta'}}{\beta_t}
  \right)^2
  \nn
  &=
  \frac{1}{\beta_t}
  \min_{a' \in [k]}
  \sumak 
  \qta 
  \left(
    {\hat{y}_{ta} - \hat{y}_{ta'}}
  \right)^2
  \nn
  &
  =
  \frac{1}{\beta_t}
  \min_{a' \in [k]}
  \sum_{a \neq a'}
  \qta 
  \left(
    {\hat{y}_{ta} - \hat{y}_{ta'}}
  \right)^2
  \com 
  \label{eq:stab_bound_global_1}
\end{align}
where the inequality follows from~\eqref{eq:xi_bound}.
Now, for any $a\in\calA$ we have 
\begin{align}
  \Expect{\hat{y}_{ta}^2}
  =
  \Expect{\prx{\frac{G(A_t, \Phi_{A_t x_t})}{\ptAt}}^2}
  \leq
  \Expect{\sumak \pta \frac{\nrm{G}_\infty^2}{\pta^2}}
  \leq
  \sumak \frac{k \nrm{G}_\infty^2}{\gamma_t}
%  =
%  \frac{k^2 \nrm{G}_\infty^2}{\gamma_t}
  =
  \frac{\cG^2}{\gamma_t}
  \com 
  \label{eq:yhat_sq_bound_global}
\end{align}
where the last inequality follows from $\pta \geq \gamma_t / k$.
Hence, using~\eqref{eq:yhat_sq_bound_global} it holds that 
% for any $a' \in \calA$ that
\begin{align}
  \Expect{
    \frac{1}{\beta_t}
    \min_{a' \in [k]}
    \sum_{a \neq a'}
    \qta 
    \left(
      {\hat{y}_{ta} - \hat{y}_{ta'}}
    \right)^2
  }
  &
  \leq 
  \Expect{
    \frac{2}{\beta_t}
    \min_{a' \in [k]}
    \sum_{a \neq a'}
    \qta 
    \frac{\cG^2}{\gamma_t}
  }
  \nn
  &
  =
  \Expect{
    \frac{2 \min_{a' \in [k]}(1 - q_{ta'}) \cG^2}{\beta_t \gamma_t}
  }
  =
  \Expect{
    \frac{2 \cG^2 b_t}{\beta_t \gamma_t}
  }
  \per 
  \label{eq:stab_bound_global_2}
\end{align}
Combining~\eqref{eq:stab_bound_global_1} and~\eqref{eq:stab_bound_global_2} yields 
\begin{align}
  \Expect{
    \linner \hat{y}_t, \qt - q_{t+1} \rinner 
    -
    D_t (q_{t+1} , \qt)
  }
  \leq 
  \Expect{
    \frac{2 \cG^2 b_t}{\beta_t \gamma_t}
  }
  \com 
  \n 
\end{align}
which completes the proof.
\begin{comment} % === penalty term === 
We next consider the penalty term.
We have
\begin{align}
  &
  \sum_{t=1}^T
  \left( \psi_t(q_{t+1}) - \psi_{t+1}(q_{t+1}) \right)
  +
  \psi_{T+1} (e_{a^*})
  -
  \psi_1(q_1)
  \nn
  &
  =
%  \sum_{a \in [k]} 
%  \left(
  \sum_{t=1}^T
  \left( \beta_t - \beta_{t+1} \right) (- H(q_{t+1}))
%  \right)
  +
  \beta_1 
%  \sumak 
  H(q_{1})
%  \nn
%  &=
%  \sum_{a \in \Pi} \left(
%  \sum_{t=1}^T
%  \left( \beta_t - \beta_{t+1} \right) h(q_{t+1,a})
%  \right)
%  -
%  \beta_1 
%  \sum_{a \in \Pi}
%  h(q_{1,a})
%  \nn
%  &
%  \nonumber
  =
  \sum_{t=1}^T
  \left( \beta_{t+1} - \beta_t \right) a_{t+1}
  +
  \beta_1 a_1
  \com
\end{align}
%where the first equality follows from $\psi(e_{a^*}) = 0$,
%the second equality follows from $h(q_{t+1, a}) = 0$ for $a \not\in \Pi$,
where we recall that $a_t$ and $b_t$ are defined in \eqref{eq:defatbtglobal}.
Combining the above arguments with Lemma~\ref{lem:FTRL} and~\eqref{eq:stab_bound_instance_global}, we have
\begin{align}
  R_T
  =
  O \left(
    \sum_{t=1}^T
    \left(
      \gamma_t
      +
      \frac{c_{\calG}^2 b_t}{\beta_t \gamma_t}
      +
      (\beta_{t+1} - \beta_t) a_{t+1}
    \right)
    +
    \beta_1 a_1
  \right)
  \com
\end{align}
which completes the proof of Lemma~\ref{lem:decompose_global}.
\end{comment}
\end{proof}
%\end{proofof}

\subsection{Proof of Proposition~\ref{prop:global}}
\begin{proof}
Note that the penalty term can be rewritten as 
\begin{align}
  &
  \sum_{t=1}^T
  \left( \psi_t(q_{t+1}) - \psi_{t+1}(q_{t+1}) \right)
  +
  \psi_{T+1} (e_{a^*})
  -
  \psi_1(q_1)
  \nn
  &
  =
  \sum_{t=1}^T
  \left( \beta_t - \beta_{t+1} \right) (- H(q_{t+1}))
  +
  \beta_1 
  H(q_{1})
  =
  \sum_{t=1}^T
  \left( \beta_{t+1} - \beta_t \right) a_{t+1}
  +
  \beta_1 a_1
  \com
  \n 
\end{align}
%where the first equality follows from $\psi(e_{a^*}) = 0$,
%the second equality follows from $h(q_{t+1, a}) = 0$ for $a \not\in \Pi$,
where we recall that the definition of $a_t$ in \eqref{eq:defatbtglobal}.
Combining this with Lemmas~\ref{lem:FTRL} and~\ref{lem:stability_bound_global}, we have
\begin{align}
  R_T
  \leq 
%  O \left(
    \underbrace{
    \sumT 
    \left(
      \gamma'_t + \frac{\cG}{2 \beta_t}
    \right)
    }_{\text{transformation term}}
    +
    \underbrace{
    \sumT
    \frac{2 c_{\calG}^2 b_t}{\beta_t \gamma_t}
    }_{\text{stability term}}
    +
    \underbrace{
    \sumT
    \left(
      (\beta_{t+1} - \beta_t) a_{t+1}
    \right)
    +
    \beta_1 a_1
    }_{\text{penalty term}}
%  \right)
  \label{eq:decompose_global}
  \com
\end{align}
where the first, second, and remaining terms correspond to the transformation, stability, and penalty terms, respectively.
We bound each term of the RHS in~\eqref{eq:decompose_global} in the following.

Note that
$b_t \leq 1$ and 
\begin{align}
  b_t = 1 - \max_{a\in[k]} \qta 
  \leq 
  - \max_{a\in[k]} \qta  \log\prx{\max_{a'\in[k]} q_{t,a'} }
  \leq 
  - \sum_{a\in[k]} \qta \log \qta
  = 
  a_t 
  \leq 
  \log \kpi 
  \com 
  \label{eq:misc_atbt}
\end{align}
where the first inequality follows from the inequality $1 - x \leq - x \log x$ for $x > 0$.
We define $z_t = \frac{a_{t+1} b_t}{\gamma'_t}$
and $Z_t = \sum_{s=1}^t z_s$.
\paragraph{[Bounding the penalty term]}
%Then,
From the definition of $\gamma'_t$, we can bound $z_t$ from below as
\begin{align}
  z_t
  =
  \frac{a_{t+1}b_t}{\gamma'_t}
  =
  \frac{4 a_{t+1}}{c_1}
  \left(
  c_1
  +
  B_t^{1/3}
  \right)
  \geq
  4 a_{t+1}
  \geq
  4 b_{t+1}
  \com 
  \label{eq:zb}
\end{align}
where the second inequality follows from $b_t \leq a_t$ in~\eqref{eq:misc_atbt}.
Further, we can bound $z_t$ from above as
\begin{align}
  z_t
  =
  \frac{4 a_{t+1}}{c_1} 
  \left(
    c_1 +
    B_t^{1/3}
  \right)
  \leq
  4
  \left(
    c_1 +
    B_t^{1/3}
  \right)
  \leq
  4 \brx{
  c_1
  +
  \left( 
    b_1
    +
    \sum_{s=1}^{t-1} z_s
  \right)^{1/3}
  }
  \leq
  8
  \left(
  c_1
  +
  Z_{t-1}
  % \sum_{s=1}^{t-1} \zeta_s
  \right)
  \com
\label{eq:zZ}
\end{align}
where the first inequality follows from
$a_{t+1} \leq \log \kpi$ and $c_1 \geq \log \kpi$,
and the second inequality follows from
$B_t = b_1 + \sum_{s=1}^{t-1} b_{s+1} \leq b_1 + \sum_{s=1}^{t-1} z_s$,
and the last inequality follows from $b_1 \leq 1 \leq c_1$.
From this, since $\beta_t$ satisfies $\beta_{t+1} - \beta_t = \frac{z_t}{a_{t+1}} \frac{c_2}{(c_1 + Z_{t-1})^{1/2}}$, we can bound the penalty term in~\eqref{eq:decompose_global} as
\begin{align}
  \nonumber
  &
  \sum_{t=1}^T 
  ( \beta_{t+1} - \beta_t ) a_{t+1}
  =
  c_2
  \sum_{t=1}^T 
  \frac{z_{t}}{ \sqrt{ c_1  +  Z_{t-1} } }
  =
  5 c_2
  \sum_{t=1}^T 
  \frac{Z_{t} - Z_{t-1}}{ 4 \sqrt{ c_1  +  Z_{t-1} } + \sqrt{c_1 + Z_{t-1}} }
  \\
  &
%  \leq
  <
  5 c_2
  \sum_{t=1}^T 
  \frac{Z_{t} - Z_{t-1}}{ \sqrt{ c_1  +  Z_{t} } + \sqrt{c_1 + Z_{t-1}} }
  =
  5 c_2
  \sum_{t=1}^T 
  \left( \sqrt{ c_1  +  Z_{t} } - \sqrt{c_1 + Z_{t-1}} \right)
  % =
  \leq
  5
  c_2
  \sqrt{ 
    Z_T
  }
  \com
  \label{eq:BB1}
\end{align}
where the first equality follows from the definitions of $\beta_t$ and $z_t$,
and the first inequality follows since 
\begin{align}
  \sqrt{c_1 + Z_t} 
  \leq 
  \sqrt{c_1 + Z_{t-1}} + \sqrt{z_t} 
  <
  % (1 + \sqrt{8})
  4
  \sqrt{c_1 + Z_{t-1}}
  \com 
  \n 
\end{align}
where the last inequality follows from~\eqref{eq:zZ}.

\paragraph{[Bounding the stability term and transformation terms]}
We define $w_t = \frac{b_t}{\gamma'_t}$ and $W_t = \sum_{s=1}^t w_s$.
From the definition of $\gamma'_t$,
we have
\begin{align}
  w_t
  =
  \frac{b_t}{\gamma'_t}
  =
  4\prx{
    1
    +
    \frac{1}{c_1}
    B_t^{1/3}
  }
  \geq 
  4
  \per
  \label{eq:boundw}
\end{align}
Using $b_t \leq 1$, we can confirm that $w_t$ satisfies
\begin{align}
  w_{1} \leq 8 \com 
  \quad
  w_{t+1} = 
  4 \prx{
  1 +
  \frac{1}{c_1}
  B_{t+1}^{1/3}
  }
  \leq 
  \prx{
  1 +
  \frac{1}{c_1}
  (B_{t} + 1)^{1/3} 
  }
  \leq
  2 w_t
  \com 
  \quad
  w_{t} 
  \leq  
  4 ( 1 + t^{1/3} )
  \per
  \label{eq:boundw2}
\end{align}
% $w$
% where the last inequality follows from $c_1 \leq c_3$.
Then $\beta_t$ can be bounded as
\begin{align}
  \beta_t
  &
%  =
  \geq 
  c_2
  +
  c_2
  \sum_{s=1}^{t-1} \frac{w_s}{ \sqrt{ c_1 + Z_{s-1} } }
  \geq
  \frac{c_2}{\sqrt{
    c_1 + Z_t
  }}
  \left(
    1
    +
    \sum_{s=1}^{t-1} w_s
  \right)
  \nn
  &
  =
  \frac{c_2}{\sqrt{
    c_1 + Z_t
  }}
  \left(
    1
    +
    W_{t-1}
  \right)
  \label{eq:beta_lower}
  \\ 
  &
  \geq
  \frac{c_2 t}{\sqrt{
    c_1 + Z_t
  }} 
  \com
  \label{eq:beta_lower_2}
\end{align}
where the second inequality follows from \eqref{eq:boundw}.

Using the above inequalities, 
we can bound the stability term in~\eqref{eq:decompose_global} as 
\begin{align}
  \sum_{t=1}^T
  \frac{b_t}{\gamma_t \beta_t}
  \leq
  \sum_{t=1}^T
  \frac{b_t}{\gamma'_t \beta_t}
  &
  \leq
  \sum_{t=1}^T
  \frac{\sqrt{c_1 + Z_t}}{c_2}
  \frac{w_t}{1 + W_{t-1}}
  \leq
  \frac{\sqrt{c_1 + Z_T}}{c_2}
  \sum_{t=1}^T
  \frac{w_t}{1 + W_{t-1}}
  \nn 
  &
  \leq
  O \left(
    \frac{\sqrt{c_1 + Z_T}}{c_2}
    \log \left(
      1 + W_{T}
    \right)
  \right)
  \leq
  O \left(
    \frac{\sqrt{c_1 + Z_T}}{c_2}
    \log T
  \right)
  \com
  \label{eq:BB2}
\end{align}
where
the first inequality follows from~\eqref{eq:beta_lower},
the last inequality follows from \eqref{eq:boundw2},
and the fourth inequality can be shown by taking the sum of the following inequality:
\begin{align*}
  \log (1 + W_{t}) - \log (1  + W_{t-1} )
  =
  \log 
    \frac{1 + W_t}{ 1 + W_{t-1}}
  =
  \log \left( 
    1 + 
    \frac{w_t}{ 1 + W_{t-1}}
  \right)
  \geq
  \frac{1}{2}
  \cdot
  \frac{w_t}{ 1 + W_{t-1}}
  \com
\end{align*}
where the inequality follows from the fact that 
$\log(1+x) \geq \frac{1}{2} x$ holds for any $x \in [0, 2]$
and that \eqref{eq:boundw2} implies
$ \frac{w_t}{ 1 + W_{t-1}} \leq \frac{w_t}{ 1 + w_t/2} \leq 2$ for all $t \in [T]$.

Using~\eqref{eq:beta_lower_2}, we can bound the second part of the transformation term in~\eqref{eq:decompose_global} as
\begin{align}
  \sum_{t=1}^T
  \frac{1}{\beta_t}
  \leq
  \sum_{t=1}^T
  \frac{ \sqrt{c_1 + Z_t}}{c_2  t}
  \leq
  \frac{ \sqrt{c_1 + Z_T}}{c_2 }
  \sum_{t=1}^T
  \frac{1}{t}
  =
  O \left(
    \frac{\sqrt{c_1 + Z_T}}{c_2 }
    \log T
  \right)
  \per
  \label{eq:BB3}
\end{align}
In addition, from the definition of $\gamma'_t$,
we can bound the remaining part of the transformation term in~\eqref{eq:decompose_global} as
\begin{align}
  \sum_{t=1}^T \gamma'_{t}
  =
  \frac{c_1}{4}
  \sum_{t=1}^T
  \frac{b_t}{c_1 + B_t^{1/3} }
  \leq
  \frac{3 c_1}{8}
  \sum_{t=1}^T
  \left(
    B_{t}^{2/3} - B_{t-1}^{2/3}
  \right)
  \leq
  \frac{3 c_1}{8}
  B_{T}^{2/3} 
  \com
  \label{eq:BB4}
\end{align}
where the first inequality follows from
$y^{2/3} - x^{2/3} \geq \frac{2}{3} ( y - x ) y^{-1/3} $,
which holds for any $y \geq x > 0$.
Combining \eqref{eq:BB1}, \eqref{eq:BB2}, \eqref{eq:BB3}, and \eqref{eq:BB4},
we can bound the right-hand side of~\eqref{eq:decompose_global} as 
\begin{align}
  &
  \sum_{t=1}^T
  \left(
    \gamma_t 
    + 
    \frac{2 \cG^2 b_t}{\gamma_t \beta_t}
    +
    (\beta_{t+1} - \beta_t) a_{t+1}
  \right)
  + \beta_1 a_1
  \nn
  &
  =
  \sum_{t=1}^T
  \left(
    \gamma'_t 
    +
    \frac{\cG}{2 \beta_t}
    +
    \frac{2 \cG^2 b_t}{\gamma_t \beta_t}
    +
    (\beta_{t+1} - \beta_t) a_{t+1}
  \right)
  + \beta_1 a_1
  \nn
  &
  =
  O\left(
    c_1 B_T^{2/3}
    +
    \left(
      \frac{ \cG^2 \log T }{c_2}
      +
      c_2
    \right)
    \sqrt{
      c_1
      +
      Z_T
    }
    + \beta_1 a_1
  \right)
  \nn
  &
  =
  O\left(
    c_1 B_T^{2/3}
    +
    \left(
      \frac{ \cG^2 \log T }{c_2}
      +
      c_2
    \right)
    \sqrt{
      c_1
      +
      \sum_{t=1}^T \frac{a_{t+1} }{c_1}
      \left(
        c_1
        +
        B_t^{1/3}
      \right)
    }
    + \beta_1 a_1
  \right)
  \nn
  &
  =
  O\left(
    c_1 B_T^{2/3}
    +
    \frac{1}{\sqrt{c_1}}
    \left(
      \frac{ \cG^2 \log T }{c_2}
      +
      c_2
    \right)
    \sqrt{
      c_1 ^2
      +
      \left(
        \log \kpi
        +
        A_T
      \right)
      \left(
        c_1
        +
        B_T^{1/3}
      \right)
    }
    + \beta_1 \log \kpi
  \right)
  \com
  \n 
\end{align}
where in the third inequality we used~\eqref{eq:zb}
and in the last equality we used $a_{T+1} = O(\log \kpi)$.
\end{proof}

\subsection{Proof of Theorem~\ref{thm:globally_obs}}\label{subsec:proof_of_theorem_globally}
%\begin{proofof}{Theorem~\ref{thm:globally_obs}}
\begin{proof}
%\begin{proof}{Theorem~\ref{thm:globally_obs}}
We define $c_1$ and $c_2$ by 
\begin{align}\label{eq:defc1c2}
  c_1 =
  \Theta \prx{
  	\prn[\big]{\cG^2 \log(T) \log(\kpi T)}^{1/3}
  }
  \quad
  \mbox{and}
  \quad
    c_2 =
    \Theta \prx{
      \sqrt{\cG^2 \log T}
    }
  \com
\end{align}
which implies that $\tilde{c} = c_1/\sqrt{\log(\kpi T)}$.
We have 
\begin{align}
  B_T
  =
  \sumT \left(1 - \max_{a\in\Pi} \qta\right)
  \leq 
  \sum_{t=1}^T
  \left(
    1 - q_{t,a^*}
  \right)
  =
  Q(a^*)
  \per 
  \label{eq:BT_bound}
\end{align}
We first consider the adversarial regime.
Since $A_T \leq T \log \kpi$ and $B_T \leq T$, using Proposition~\ref{prop:global} we have
\begin{align}
  R_T 
  &=
  O\prx{
      c_1 T^{2/3} + \tilde{c} \sqrt{c_1^2 + (\log \kpi + T \log \kpi) (c_1 + T^{1/3})}
      +
      \beta_1 \log \kpi
  }
  \nn
  &=
  O\prx{
    \prx{c_1 + \tilde{c} \sqrt{\log \kpi}}
    T^{2/3}
    +
    \sqrt{\frac{\log \kpi}{\log(\kpi T)}} c_1^{3/2} T^{1/2}
    +
    \frac{c_1^2}{\sqrt{\log(\kpi T)}}
    +
    \beta_1 \log \kpi
  }
  \per
  \label{eq:regret_bound_global_adv_1}
\end{align}
We next consider the adversarial regime with a self-bounding constraint.
When ${Q}(a^*) \le c_1^3$ we can show that the obtained bound is smaller than the desired bound as follows.
When $Q(a^*) \le \e \le c_1^3$, using Lemma~\ref{lem:entropy_bound} and~\eqref{eq:BT_bound}, we have $A_T \le \e\log(\kpi T)$ and $B_T \le \e$.
Hence, from Proposition~\ref{prop:global}, we have
\begin{align}
  R_T
  =
  O \prx{
    c_1 + \tilde{c} \sqrt{c_1^2 + \log(\kpi T) c_1}
    +
    \beta_1 \log \kpi
  }
  =
  O \prx{
    \frac{c_1^2}{\sqrt{\log(\kpi T)}}
    +
    \beta_1 \log \kpi
  }
  =
  O \prx{
    c_1^3
  }
  \per
  \n 
\end{align}
When $\e < Q(a^*) \le c_1^3$, % using Lemma~\ref{lem:boundATBT} 
using Lemma~\ref{lem:entropy_bound} and~\eqref{eq:BT_bound}
we have $A_T \le c_1^3 \log(\kpi T)$ and $B_T \le c_1^3$.
Hence, from Proposition~\ref{prop:global}, we have
\begin{align}
  R_T
  &=
  O \prx{
    c_1^3 + \tilde{c} \sqrt{c_1^2 + \prx{\log\kpi + c_1^3  \log(\kpi T)} c_1}
    +
    \beta_1 \log \kpi
  }
  \nn  
  &=
  O \prx{
    {\cG^2 \log(T) \log(\kpi T)}
  }
  =
  O \prx{
    c_1^3
  }
  \per
  \n 
\end{align}
Hence, we only need to consider the case of $Q(a^*) > c_1^3$ in the following.
Since $Q(a^*) \ge \e$ we have $A_T \le Q(a^*) \log(\kpi T)$.
Using Proposition~\ref{prop:global} with this inequality, Lemma~\ref{lem:entropy_bound}, and~\eqref{eq:BT_bound}, we have
\begin{align}
  R_T
  &
  =
  O \left(
    \E \left[
    c_1 {Q}(a^*)^{2/3}
    +
    \tilde{c} 
    \sqrt{
    c_1^2 
    + 
    \prn[\big]{\log \kpi + {Q}(a^*)\log(\kpi T)}
    \left(c_1 + {Q}(a^*)^{1/3} \right) }
    \right]
    +
    \beta_1 \log \kpi
  \right)
  \nn
  &
  \le 
  O \left(
    \E \left[
    c_1 {Q}(a^*)^{2/3}
    +
    \tilde{c} 
    \sqrt{ {Q}(a^*) \log(\kpi T) {Q}(a^*)^{1/3} }
    \right]
  \right)
  \nn 
  &\leq
  O \left(
    \left(
    c_1 
    +
    \tilde{c} 
    \sqrt{\log(\kpi T)}
    \right)
    \bar{Q}(a^*)^{2/3}
  \right)
  \com
  \label{eq:regret_bound_global_ARSBC_1}
\end{align}
where the first inequality follows from $Q(a^*) > c_1^3$,
and the second inequality follows from Jensen's inequality.
Hence, by~\eqref{eq:regret_bound_global_adv_1} and~\eqref{eq:regret_bound_global_ARSBC_1},
there exists 
$
%\begin{align}
  \hat{c} = O \left( c_1 + \tilde{c}\sqrt{ \log(\kpi T)} \right)
%  \per
%\end{align}
$ satisfying
and $R_T \le \hat{c} \, \bar{Q}(a^*)^{2/3}$ for the adversarial regime with a self-bounding constraint
and $R_T \le \hat{c} \, T^{2/3}$ for the adversarial regime.

Now, by recalling the definitions of $c_1$ and $c_2$ in~\eqref{eq:defc1c2}, we have
\begin{align}
  \hat{c} 
  &= 
  O\prx{
  \prn[\big]{\cG^2 \log(T) \log(\kpi T)}^{1/3}
  +
  \frac{1}{\sqrt{c_1}} \left( \frac{\cG^2 \log T}{c_2} + {c_2} \right) \sqrt{\log(\kpi T)}
  }
  \nn
  &=
  O \prx{
  	(\cG^2 \log(T) \log(\kpi T))^{1/3}
  }
  \com
  \label{eq:final_c_hat}
\end{align}
which gives the desired bounds for the adversarial regime.

%For the stochastic regime, 
For the adversarial regime with a self-bounding constraint, 
from the above inequality $R_T \le \hat{c} \, \bar{Q}(a^*)^{2/3}$ and Lemma~\ref{lem:selfQ} with $c = 1/2 \leq 1 - \gamma_t$, we have for any $\lambda \in (0, 1]$ that
\begin{align}
  R_T
  &=
  (1 + \lambda) R_T - \lambda R_T
%  \nn
%  &
  \le
  (1 + \lambda) \hat{c} \cdot \bar{Q}(a^*)^{2/3}
  -
  \frac\lambda2 \Deltamin \bar{Q}(a^*)
  +
  \lambda C
%  =
%  \left( \frac{2^3 \hat{c}^3}{\Deltamin^2 } \right)^{1/3}
%  \left( \Deltamin \bar{Q}(a^*) \right)^{2/3}
%  -
%  \frac12 \Deltamin \bar{Q}(a^*)
  \nn
  &
  \leq 
  O \left(
  \frac{(1 + \lambda)^3 \hat{c}^3}{ \lambda^2 \Deltamin^2 }
  \right)
  + \lambda C
%  \nn 
%  &
  =
  O \left(
    \left(1 + \frac{1}{\lambda^2} \right) \frac{\hat{c}^3}{\Deltamin^2}
  \right) 
  + \lambda C
  \com
  \label{eq:regret_bound_global_1}
\end{align}
where the first inequality follows from the inequality
%$a x^{2/3} - (x/2) \le 16 a^3/27$ for $a > 0$.
$a x^{2/3} - b(x/2) \leq 16a^3/(27b^2)$ for $a, b > 0$,
and the last equality follows since $\lambda \in (0,1]$.
Combining~\eqref{eq:final_c_hat} and~\eqref{eq:regret_bound_global_1},
and taking $\lambda = O\left( \frac{\cG^2 \log(T) \log(\kpi T)}{C \Deltamin^2} \right)$, we have the desired result for the adversarial regime with a self-bounding constraint.
%\end{proofof}
\end{proof}

\section{Regret Bounds when the Optimization Problem is Not Exactly Solved}\label{sec:eps_opt_bound}
This section discusses the regret bound when the optimization problem~\eqref{eq:optprime} is not exactly solved, on which a similar discussion is given in~\citet[Chapter 37]{lattimore2020book}.
We say that the optimization problem~\eqref{eq:optprime} can be solved with precision $\ep \geq 0$, if 
we can obtain $G \in \calH$ and $p \in \calP'_k(q)$ such that 
\begin{align}
  \max_{x \in [d]} \Bigg[
  \frac{(p-q)^\top \lossmat e_x + \bias_q(G ; x)}{\eta} 
  + 
  \frac{1}{\eta^2} \sumak 
  p_a 
  \innerprod{q}{\xi\left( \frac{\eta G(a, \fbmat_{ax})}{p_a} \right)}
  \Bigg] 
  \le
  \opt'_q(\eta) + \epsilon
  \per
  \n 
\end{align}
Then if we run Algorithm~\ref{alg:MixedExp3Locally} solving~\eqref{eq:optprime} with precision $\ep$,
one can see that we can obtain the following regret bounds.
For the adversarial regime with a $(\Delta, C, T)$ self-bounding constraint, we have
\begin{align}
  R_T
  =
  O\prx{
    \frac{\prn[\big]{m k^2 + {\ep^2}/{(m k)}}^2 \log(T) \log(\kpi T)}{\Deltamin}
    +
    \sqrt{\frac{C \prn[\big]{m k^2 + {\ep^2}/{(m k)}}^2 \log(T) \log(\kpi T)}{\Deltamin}
    }
  }
  \com
  \n 
\end{align}
and for the adversarial regime, we have
\begin{align}
  R_T
  =
  O\prx{
    m k^{3/2} \sqrt{T \log(T) \log \kpi}
    +
    \epsilon \, \frac{\sqrt{T \log(\kpi) \log(T)}}{mk^{3/2}}
  }
  \per
  \n 
\end{align}

Here, we give an overview of the analysis.
Considering that the optimization problem in~\eqref{eq:optprime} can be solved with precision $\ep \geq 0$,
the RHS of \eqref{eq:equality_local} can be replaced with $3m^2 k^3 + \ep$.
%  Then the RHS in~\eqref{eq:sum_etaV_bound} can be replaced with $(3m^2k^3 + \ep) \sumT \eta_t$. 
Then a similar analysis as the proof of Theorem~\ref{thm:locally_obs} leads to 
\begin{align}
  R_T
  \leq 
  O\prx{
    \prx{mk^{3/2} + \frac{\ep}{mk^{3/2}}}
    \sqrt{\log(T) \sumT H(q_t)}
  }  
  \per 
  \n 
\end{align}
Using $\sumT H(q_t) \leq T \log \kpi$ gives the bound for the adversarial regime.
Replacing $m^2 k^{4}$ with $\prx{mk^2 + \frac{\ep}{mk}}^2$ in~\eqref{eq:regert_selfbound_local} and appropriately choose $\lambda$ (note that we can take $\lambda$ depending on $\ep$), we obtain the desired bound for the adversarial regime with a self-bounding constraint.
%\end{proofsketchof}

%% file: alt2023-main.bbl
\begin{thebibliography}{36}
\providecommand{\natexlab}[1]{#1}
\providecommand{\url}[1]{\texttt{#1}}
\expandafter\ifx\csname urlstyle\endcsname\relax
  \providecommand{\doi}[1]{doi: #1}\else
  \providecommand{\doi}{doi: \begingroup \urlstyle{rm}\Url}\fi

\bibitem[Alon et~al.(2015)Alon, Cesa-Bianchi, Dekel, and Koren]{alon2015online}
Noga Alon, Nicolò Cesa-Bianchi, Ofer Dekel, and Tomer Koren.
\newblock Online learning with feedback graphs: Beyond bandits.
\newblock In \emph{Proceedings of The 28th Conference on Learning Theory},
  volume~40 of \emph{Proceedings of Machine Learning Research}, pages 23--35.
  PMLR, 2015.

\bibitem[Auer et~al.(2002)Auer, Cesa-Bianchi, Freund, and
  Schapire]{auer2002nonstochastic}
Peter Auer, Nicolo Cesa-Bianchi, Yoav Freund, and Robert~E Schapire.
\newblock The nonstochastic multiarmed bandit problem.
\newblock \emph{SIAM Journal on Computing}, 32\penalty0 (1):\penalty0 48--77,
  2002.

\bibitem[Bart{\'o}k(2013)]{bartok13near}
G{\'a}bor Bart{\'o}k.
\newblock A near-optimal algorithm for finite partial-monitoring games against
  adversarial opponents.
\newblock In \emph{Proceedings of the 26th Annual Conference on Learning
  Theory}, volume~30 of \emph{Proceedings of Machine Learning Research}, pages
  696--710. PMLR, 2013.

\bibitem[Bart{\'{o}}k et~al.(2012)Bart{\'{o}}k, Zolghadr, and
  Szepesv{\'{a}}ri]{Bartok12CBP}
G{\'{a}}bor Bart{\'{o}}k, Navid Zolghadr, and Csaba Szepesv{\'{a}}ri.
\newblock An adaptive algorithm for finite stochastic partial monitoring.
\newblock In \emph{the 29th International Conference on Machine Learning},
  pages 1--20, 2012.

\bibitem[Bartók et~al.(2011)Bartók, Pál, and Szepesvári]{Bartok11minimax}
Gábor Bartók, Dávid Pál, and Csaba Szepesvári.
\newblock Minimax regret of finite partial-monitoring games in stochastic
  environments.
\newblock In \emph{the 24th Annual Conference on Learning Theory}, volume~19,
  pages 133--154, 2011.

\bibitem[Bubeck and Slivkins(2012)]{bubeck2012best}
S{\'e}bastien Bubeck and Aleksandrs Slivkins.
\newblock The best of both worlds: Stochastic and adversarial bandits.
\newblock In \emph{Proceedings of the 25th Annual Conference on Learning
  Theory}, volume~23 of \emph{Proceedings of Machine Learning Research}, pages
  42.1--42.23. PMLR, 2012.

\bibitem[Cesa{-}Bianchi et~al.(2006)Cesa{-}Bianchi, Lugosi, and
  Stoltz]{CesaBianchi06regret}
Nicol{\`{o}} Cesa{-}Bianchi, G{\'{a}}bor Lugosi, and Gilles Stoltz.
\newblock Regret minimization under partial monitoring.
\newblock \emph{Mathematics of Operations Research}, 31\penalty0 (3):\penalty0
  562--580, 2006.

\bibitem[Erez and Koren(2021)]{erez2021best}
Liad Erez and Tomer Koren.
\newblock Towards best-of-all-worlds online learning with feedback graphs.
\newblock In \emph{Advances in Neural Information Processing Systems},
  volume~34, pages 28511--28521. Curran Associates, Inc., 2021.

\bibitem[Foster and Rakhlin(2012)]{foster12no}
Dean Foster and Alexander Rakhlin.
\newblock No internal regret via neighborhood watch.
\newblock In \emph{Proceedings of the Fifteenth International Conference on
  Artificial Intelligence and Statistics}, volume~22 of \emph{Proceedings of
  Machine Learning Research}, pages 382--390. PMLR, 2012.

\bibitem[Gaillard et~al.(2014)Gaillard, Stoltz, and van
  Erven]{gaillard2014second}
Pierre Gaillard, Gilles Stoltz, and Tim van Erven.
\newblock A second-order bound with excess losses.
\newblock In \emph{Proceedings of The 27th Conference on Learning Theory},
  volume~35 of \emph{Proceedings of Machine Learning Research}, pages 176--196.
  PMLR, 2014.

\bibitem[Honda and Takemura(2011)]{Honda11dmed}
Junya Honda and Akimichi Takemura.
\newblock An asymptotically optimal policy for finite support models in the
  multiarmed bandit problem.
\newblock \emph{Machine Learning}, 85\penalty0 (3):\penalty0 361--391, 2011.
\newblock ISSN 1573-0565.
\newblock \doi{10.1007/s10994-011-5257-4}.

\bibitem[Huang et~al.(2022)Huang, Dai, and Huang]{huang22adaptive}
Jiatai Huang, Yan Dai, and Longbo Huang.
\newblock Adaptive best-of-both-worlds algorithm for heavy-tailed multi-armed
  bandits.
\newblock In \emph{Proceedings of the 39th International Conference on Machine
  Learning}, volume 162 of \emph{Proceedings of Machine Learning Research},
  pages 9173--9200. PMLR, 2022.

\bibitem[Ito(2021)]{ito2021hybrid}
Shinji Ito.
\newblock Hybrid regret bounds for combinatorial semi-bandits and adversarial
  linear bandits.
\newblock In \emph{Advances in Neural Information Processing Systems},
  volume~34, pages 2654--2667. Curran Associates, Inc., 2021.

\bibitem[Ito et~al.(2022{\natexlab{a}})Ito, Tsuchiya, and Honda]{ito2022nearly}
Shinji Ito, Taira Tsuchiya, and Junya Honda.
\newblock Nearly optimal best-of-both-worlds algorithms for online learning
  with feedback graphs.
\newblock \emph{arXiv preprint arXiv:2206.00873}, 2022{\natexlab{a}}.

\bibitem[Ito et~al.(2022{\natexlab{b}})Ito, Tsuchiya, and
  Honda]{ito22adversarially}
Shinji Ito, Taira Tsuchiya, and Junya Honda.
\newblock Adversarially robust multi-armed bandit algorithm with
  variance-dependent regret bounds.
\newblock In \emph{Proceedings of Thirty Fifth Conference on Learning Theory},
  volume 178 of \emph{Proceedings of Machine Learning Research}, pages
  1421--1422. PMLR, 2022{\natexlab{b}}.

\bibitem[Jin and Luo(2020)]{jin2020simultaneously}
Tiancheng Jin and Haipeng Luo.
\newblock Simultaneously learning stochastic and adversarial episodic {MDP}s
  with known transition.
\newblock In \emph{Advances in Neural Information Processing Systems},
  volume~33, pages 16557--16566. Curran Associates, Inc., 2020.

\bibitem[Komiyama et~al.(2015)Komiyama, Honda, and Nakagawa]{Komiyama15PMDEMD}
Junpei Komiyama, Junya Honda, and Hiroshi Nakagawa.
\newblock Regret lower bound and optimal algorithm in finite stochastic partial
  monitoring.
\newblock In \emph{Advances in Neural Information Processing Systems 28}, pages
  1792--1800. Curran Associates, Inc., 2015.

\bibitem[Kong et~al.(2022)Kong, Zhou, and Li]{kong22simultaneously}
Fang Kong, Yichi Zhou, and Shuai Li.
\newblock Simultaneously learning stochastic and adversarial bandits with
  general graph feedback.
\newblock In \emph{Proceedings of the 39th International Conference on Machine
  Learning}, volume 162 of \emph{Proceedings of Machine Learning Research},
  pages 11473--11482. PMLR, 2022.

\bibitem[Lattimore and Szepesv{\'a}ri(2019{\natexlab{a}})]{lattimore19cleaning}
Tor Lattimore and Csaba Szepesv{\'a}ri.
\newblock Cleaning up the neighborhood: A full classification for adversarial
  partial monitoring.
\newblock In \emph{Proceedings of the 30th International Conference on
  Algorithmic Learning Theory}, volume~98 of \emph{Proceedings of Machine
  Learning Research}, pages 529--556. PMLR, 2019{\natexlab{a}}.

\bibitem[Lattimore and
  Szepesv{\'a}ri(2019{\natexlab{b}})]{lattimore2019information}
Tor Lattimore and Csaba Szepesv{\'a}ri.
\newblock An information-theoretic approach to minimax regret in partial
  monitoring.
\newblock In \emph{Proceedings of the Thirty-Second Conference on Learning
  Theory}, volume~99 of \emph{Proceedings of Machine Learning Research}, pages
  2111--2139. PMLR, 2019{\natexlab{b}}.

\bibitem[Lattimore and Szepesv{\'a}ri(2020{\natexlab{a}})]{lattimore2020book}
Tor Lattimore and Csaba Szepesv{\'a}ri.
\newblock \emph{Bandit algorithms}.
\newblock Cambridge University Press, 2020{\natexlab{a}}.

\bibitem[Lattimore and
  Szepesv{\'a}ri(2020{\natexlab{b}})]{lattimore20exploration}
Tor Lattimore and Csaba Szepesv{\'a}ri.
\newblock Exploration by optimisation in partial monitoring.
\newblock In \emph{Proceedings of Thirty Third Conference on Learning Theory},
  volume 125 of \emph{Proceedings of Machine Learning Research}, pages
  2488--2515. PMLR, 2020{\natexlab{b}}.

\bibitem[Lee et~al.(2021)Lee, Luo, Wei, Zhang, and Zhang]{lee2021achieving}
Chung-Wei Lee, Haipeng Luo, Chen-Yu Wei, Mengxiao Zhang, and Xiaojin Zhang.
\newblock Achieving near instance-optimality and minimax-optimality in
  stochastic and adversarial linear bandits simultaneously.
\newblock In \emph{Proceedings of the 38th International Conference on Machine
  Learning}, volume 139 of \emph{Proceedings of Machine Learning Research},
  pages 6142--6151. PMLR, 2021.

\bibitem[Luo and Schapire(2015)]{luo2015achieving}
Haipeng Luo and Robert~E. Schapire.
\newblock Achieving all with no parameters: {AdaNormalHedge}.
\newblock In Peter Grünwald, Elad Hazan, and Satyen Kale, editors,
  \emph{Proceedings of The 28th Conference on Learning Theory}, volume~40 of
  \emph{Proceedings of Machine Learning Research}, pages 1286--1304. PMLR,
  2015.

\bibitem[Lykouris et~al.(2018)Lykouris, Mirrokni, and
  Paes~Leme]{lykouris2018stochastic}
Thodoris Lykouris, Vahab Mirrokni, and Renato Paes~Leme.
\newblock Stochastic bandits robust to adversarial corruptions.
\newblock In \emph{Proceedings of the 50th Annual ACM SIGACT Symposium on
  Theory of Computing}, pages 114--122, 2018.

\bibitem[McMahan(2011)]{mcmahan2011follow}
Brendan McMahan.
\newblock Follow-the-regularized-leader and mirror descent: Equivalence
  theorems and {L1} regularization.
\newblock In \emph{Proceedings of the Fourteenth International Conference on
  Artificial Intelligence and Statistics}, volume~15 of \emph{Proceedings of
  Machine Learning Research}, pages 525--533. PMLR, 2011.

\bibitem[Piccolboni and Schindelhauer(2001)]{Piccolboni01FeedExp3}
Antonio Piccolboni and Christian Schindelhauer.
\newblock Discrete prediction games with arbitrary feedback and loss (extended
  abstract).
\newblock In \emph{Computational Learning Theory}, pages 208--223, 2001.

\bibitem[Rouyer et~al.(2022)Rouyer, van~der Hoeven, Cesa-Bianchi, and
  Seldin]{rouyer2022near}
Chlo{\'e} Rouyer, Dirk van~der Hoeven, Nicol{\`o} Cesa-Bianchi, and Yevgeny
  Seldin.
\newblock A near-optimal best-of-both-worlds algorithm for online learning with
  feedback graphs.
\newblock \emph{arXiv preprint arXiv:2206.00557}, 2022.

\bibitem[Rustichini(1999)]{Rustichini99general}
Aldo Rustichini.
\newblock Minimizing regret: The general case.
\newblock \emph{Games and Economic Behavior}, 29\penalty0 (1):\penalty0
  224--243, 1999.

\bibitem[Saha and Gaillard(2022)]{saha22versatile}
Aadirupa Saha and Pierre Gaillard.
\newblock Versatile dueling bandits: Best-of-both world analyses for learning
  from relative preferences.
\newblock In \emph{Proceedings of the 39th International Conference on Machine
  Learning}, volume 162 of \emph{Proceedings of Machine Learning Research},
  pages 19011--19026. PMLR, 2022.

\bibitem[Thompson(1933)]{Thompson1933likelihood}
William~R Thompson.
\newblock On the likelihood that one unknown probability exceeds another in
  view of the evidence of two samples.
\newblock \emph{Biometrika}, 25\penalty0 (3-4):\penalty0 285--294, 12 1933.

\bibitem[Tsuchiya et~al.(2020)Tsuchiya, Honda, and
  Sugiyama]{Tsuchiya20analysis}
Taira Tsuchiya, Junya Honda, and Masashi Sugiyama.
\newblock Analysis and design of {T}hompson sampling for stochastic partial
  monitoring.
\newblock In \emph{Advances in Neural Information Processing Systems},
  volume~33, pages 8861--8871. Curran Associates, Inc., 2020.

\bibitem[Vanchinathan et~al.(2014)Vanchinathan, Bart\'{o}k, and
  Krause]{Vanchinathan14BPM}
Hastagiri~P Vanchinathan, G\'{a}bor Bart\'{o}k, and Andreas Krause.
\newblock Efficient partial monitoring with prior information.
\newblock In \emph{Advances in Neural Information Processing Systems 27}, pages
  1691--1699. Curran Associates, Inc., 2014.

\bibitem[Wei and Luo(2018)]{wei2018more}
Chen-Yu Wei and Haipeng Luo.
\newblock More adaptive algorithms for adversarial bandits.
\newblock In \emph{Proceedings of the 31st Conference On Learning Theory},
  volume~75 of \emph{Proceedings of Machine Learning Research}, pages
  1263--1291. PMLR, 2018.

\bibitem[Zimmert and Seldin(2021)]{zimmert2021tsallis}
Julian Zimmert and Yevgeny Seldin.
\newblock Tsallis-{INF}: An optimal algorithm for stochastic and adversarial
  bandits.
\newblock \emph{Journal of Machine Learning Research}, 22\penalty0
  (28):\penalty0 1--49, 2021.

\bibitem[Zimmert et~al.(2019)Zimmert, Luo, and Wei]{zimmert2019beating}
Julian Zimmert, Haipeng Luo, and Chen-Yu Wei.
\newblock Beating stochastic and adversarial semi-bandits optimally and
  simultaneously.
\newblock In \emph{Proceedings of the 36th International Conference on Machine
  Learning}, volume~97 of \emph{Proceedings of Machine Learning Research},
  pages 7683--7692. PMLR, 2019.

\end{thebibliography}
